\documentclass{article}

\usepackage{algorithm}
\usepackage{algorithmicx}
\usepackage[noend]{algpseudocode}
\usepackage{graphicx}
\usepackage[preprint]{neurips_2025}

\usepackage[utf8]{inputenc} 
\usepackage[T1]{fontenc}    
\usepackage{hyperref}       
\usepackage{url}            
\usepackage{booktabs,enumitem}       
\usepackage{amsfonts}       
\usepackage{nicefrac}       
\usepackage[dvipsnames]{xcolor}
\usepackage{algorithm}
\usepackage{algpseudocode}
\usepackage{tablefootnote}
\hypersetup{colorlinks,citecolor=Cerulean,linkcolor=red}
\usepackage{colortbl}
\usepackage[nohints]{minitoc}

\title{Dimension Agnostic Testing of Survey Data Credibility through the Lens of Regression}

\author{%
  Debabrota Basu\\
  \'Equipe Scool, Univ. Lille, Inria,\\ 
  CNRS, Centrale Lille, UMR 9189- CRIStAL\\ 
  F-59000 Lille, France\\
  \And
  Sourav Chakraborty \\
  Indian Statistical Institute \\
  Kolkata, India\\
  \And
  Debarshi Chanda \\
  Indian Statistical Institute \\
  Kolkata, India\\
  \And
  Buddha Dev Das \\
  Indian Statistical Institute \\
  Kolkata, India\\
  \And
  Arijit Ghosh \\
  Indian Statistical Institute \\
  Kolkata, India\\
  \And
  Arnab Ray \\
  Indian Statistical Institute \\
  Kolkata, India\\
}

\usepackage{float}

\usepackage{amsmath, mathtools, amsthm, amssymb, times, yhmath}

\usepackage{wrapfig, epsfig, graphicx, pgf, tikz, pgfplots}
\usepackage{ulem}

\usetikzlibrary{calc, shapes.geometric, arrows, automata}
\pgfplotsset{compat=1.18}

\usepackage{enumerate, thmtools, soul, setspace}

\usepackage[textwidth=4em]{todonotes}

\theoremstyle{plain}
\newtheorem{theorem}{Theorem}
\newtheorem{lemma}[theorem]{Lemma}

\theoremstyle{definition}
\newtheorem{definition}[theorem]{Definition}

\newtheorem{assumption}{Assumption}

\newtheorem{remark}[theorem]{Remark}

\newcommand{\field}[1]{\mathbb{#1}}
\newcommand{\R}{\field{R}}

\newcommand\inn[2]{ \left\langle {#1} \,,\, {#2} \right\rangle }

\newcommand{\norm}[1]{\left\|{#1}\right\|}
\newcommand{\func}[3]{{#1} : {#2} \rightarrow {#3}}
\DeclarePairedDelimiter{\ceil}{\lceil}{\rceil}

\newcommand{\defeq}{\stackrel{\rm def}{=}}

\newcommand{\blue}[1]{\textcolor{blue}{#1}}
\newcommand{\red}[1]{\textcolor{red}{#1}}

\newcommand{\fbrac}[1]{\left({#1}\right)}
\newcommand{\sbrac}[1]{\left\{{#1}\right\}}
\newcommand{\tbrac}[1]{\left[{#1}\right]}
\newcommand{\abs}[1]{\left|{#1}\right|}

\newcommand{\argmin}{\mathop{\mathrm{argmin}}}

\DeclareMathOperator*{\E}{\field{E}}
\DeclareMathOperator*{\Var}{\mathrm{Var}}

\newcommand{\distribution}{\sD}

\newcommand{\normal}{\sN}

\newcommand{\approxerror}{\epsilon}
\newcommand{\confidence}{\delta}

\newcommand{\bigomega}[1]{\Omega\fbrac{{#1}}}

\newcommand{\sD}{\mathcal{D}}

\newcommand{\sF}{\mathcal{F}}
\newcommand{\sG}{\mathcal{G}}

\newcommand{\sL}{\mathcal{L}}
\newcommand{\sM}{\mathcal{M}}
\newcommand{\sN}{\mathcal{N}}

\newcommand{\sR}{\mathcal{R}}

\newcommand{\sX}{\mathcal{X}}
\newcommand{\sY}{\mathcal{Y}}
\newcommand{\sZ}{\mathcal{Z}}

\renewcommand{\d}{d} 
\newcommand{\n}{m}

\newcommand{\datapoint}{z}

\newcommand{\datadomain}{\sZ}
\newcommand{\xdomain}{\sX}
\newcommand{\yrange}{\sY}
\newcommand{\distname}{\texttt{FDD}}
\newcommand{\sbound}{\sM}

\newcommand{\covariatei}{x} 
\newcommand{\covariate}{\mathbf{\covariatei}} 
\newcommand{\covariatematrix}{\MakeUppercase{\covariate}}
\newcommand{\covariatedomain}{\R^\d}    

\newcommand{\boundonx}{\zeta} 

\newcommand{\response}{y}  

\newcommand{\responserange}{\R}    

\newcommand{\coeff}{\boldsymbol{\theta}} 

\newcommand{\truecoeff}{\coeff^*}
\newcommand{\empcoeff}{\widehat{\coeff}}
\newcommand{\sparsity}{s}

\newcommand{\regnoise}{\eta}    

\newcommand{\regnoisesd}{\sigma_\regnoise}
\newcommand{\regnoisevariance}{\regnoisesd^2}

\newcommand{\kernel}{\mathbf{K}}
\newcommand{\kernelfunction}{\mathbf{\Phi}}
\newcommand{\hilbertspace}{\mathbb{H}}
\newcommand{\kernelset}{\regressionset_{\kernel}}

\newcommand{\kernelnormbound}{r}

\newcommand{\lasso}{\texttt{Lasso}}
\newcommand{\ridge}{\texttt{Ridge}}
\newcommand{\kerna}{\texttt{Kernel}}

\newcommand{\hypothesisset}{\regressionset}
\newcommand{\hypothesis}{\regression}
\newcommand{\losshypothesisset}{\sG}
\newcommand{\losshypothesis}{g}
\newcommand{\regressionset}{\sF}
\newcommand{\regression}{f}
\newcommand{\boundedlinfunc}{\regressionset}
\newcommand{\lassoset}{\regressionset_1}
\newcommand{\ridgeset}{\regressionset_2}
\newcommand{\empradcomp}[2]{\widehat{\Re}_{#1}\fbrac{#2}}
\newcommand{\radcomp}[1]{\Re\fbrac{#1}}

\newcommand{\totalsampD}{\tau}

\newcommand{\loss}{\sL} 
\newcommand{\lossbound}{M}
\newcommand{\losslipschitz}{\mu}

\newcommand{\Lspace}{\texttt{L}}

\DeclareMathAlphabet\mathbfcal{OMS}{cmsy}{b}{n}
\newcommand{\rademacher}{r}
\newcommand{\distset}{\mathbfcal{D}}

\newcommand{\sampleD}{t}

\newcommand{\linweight}{1}

\newcommand{\sampD}{\sD_{\surveyset}}
\newcommand{\D}{\sD^*}
\newcommand{\valD}{\D}
\newcommand{\tol}{\epsilon}
\newcommand{\distF}{\texttt{dist}_{\D}}
\newcommand{\distltwo}{\texttt{dist}_{\sD}}

\newcommand{\optimalf}{f^*}
\newcommand{\surveyf}{f_\surveyset}
\newcommand{\optimalsurveyf}{f^*_\surveyset}
\newcommand{\sizeS}{m}

\newcommand{\estimatevarD}{\hat{\gamma}}

\newcommand{\SurVerif}{\ensuremath{\texttt{SurVerify}}}

\newcommand{\surveyset}{S}

\newcommand{\errcoeff}{\coeff_{e}}

\newcommand{\mineigen}{\lambda_{\min}}

\newcommand{\acsincome}{\texttt{ACS\_Income}}

\usepackage{tcolorbox}
\newtcolorbox{note}[1][]
{
coltitle=black,
colbacktitle=green,
colback=green!10,
arc=1pt,
boxrule=1pt,
title=#1 
}

\begin{document}

\maketitle
\doparttoc
\faketableofcontents

\begin{abstract}
  Assessing whether a sample survey credibly represents the population is a critical question for ensuring the validity of downstream research. Generally, this problem reduces to estimating the distance between two high-dimensional distributions, which typically requires a number of samples that grows exponentially with the dimension. However, depending on the model used for data analysis, the conclusions drawn from the data may remain consistent across different underlying distributions. In this context, we propose a task-based approach to assess the credibility of sampled surveys. Specifically, we introduce a model-specific distance metric to quantify this notion of credibility. We also design an algorithm to verify the credibility of survey data in the context of regression models. Notably, the sample complexity of our algorithm is independent of the data’s dimension. This efficiency stems from the fact that the algorithm focuses on verifying the credibility of the survey data rather than reconstructing the underlying regression model. Furthermore, we show that if one attempts to verify credibility by reconstructing the regression model, the sample complexity scales linearly with the dimensionality of the data. 
 We prove the theoretical correctness of our algorithm and numerically demonstrate our algorithm's performance.
\end{abstract} 

\section{Introduction}\label{section:introduction}

Socio-economic surveys are conducted globally to collect data on population characteristics for a variety of purposes, including demographic and economic analyses, educational planning, poverty assessments, exit poll evaluations, and measuring progress toward national goals~\citep{Groves2011survey,Kenny2021_US-Census}.
The primary aim of many surveys is to support inference-driven analyses that uncover patterns to inform future research and policy decisions~\citep{heeringa2017applied,statscanada}, as well as to monitor and evaluate the long-term impacts of various policies~\citep{banerjee2020governance}. These survey datas serve as long-term benchmarks for validating research hypotheses~\citep{salant1994conduct,statscanada}. Therefore, verifying the credibility of such survey data is essential to ensure the validity of downstream analyses.

Ideally, properly collected data should be a faithful representation of the population, and representative data should ensure the validity of subsequent research. 
However, in practice, survey data rarely reflect the population perfectly~\citep{Maul03042017,Isakov2020Towards}. In the social sciences, it is rare to find large-scale surveys that do not employ stratified or multistage sampling techniques~\citep{Groves2011survey,Lohr2021_survey-sampling,Kalton2020_survey-sampling}.
In practice, these surveys are often carried out under logistical constraints. 

Determining whether a collected sample accurately represents the population is a longstanding challenge in both statistics and computer science—often framed in the latter as the problem of measuring the closeness between two distributions~\citep{Batu2000_equivalence,CanonneTopicsDT2022}. In other words, verifying representativeness is inherently inefficient and resource-intensive. In many cases, data collectors do not even claim that their samples are representative. Nevertheless, such data are routinely used for population-level research. Naturally, this raises the question of how much trust one can place in the resulting analyses. The answer hinges on the ``credibility'' of the data.
In this paper, we propose a principled approach to quantify the credibility of survey data, along with an efficient method for doing so.

A key observation is that if the goal is merely to ensure the validity of research conducted using the data, then verifying whether the data is fully representative of the population may be unnecessary—or even excessive. In such cases, traditional methods for assessing representativeness may be too rigid or resource-intensive to be practical.
Specifically, if the analysis relies on a well-established class of inference tools, \textit{we should be able to certify that any conclusions drawn using these tools from the given survey are valid, regardless of whether the data perfectly mirrors the population.}

 

One widely used and interpretable method for analyzing survey data is fitting a regression model. For example, \cite{Balia2008_survey} utilizes data from the British Health and Lifestyle Survey (1984-1985) and its longitudinal follow-up in May 2003 to demonstrate a strong association between mortality and socio-economic status.
Motivated by such applications, in this paper, we ask the following question:
\begin{center}
\textit{Can we verify whether the conclusions drawn from a regression model fitted on a given survey dataset would yield similar results if applied to the entire population?}
\end{center}


Conducting large-scale sample surveys is often complex and costly, which can result in compromised data quality. However, it is commonly assumed that collecting a small number of additional high-quality data points can help validate the overall dataset.
Building on this idea, our approach to the question above involves leveraging a limited amount of high-quality supplementary sample—alongside the original survey data—to assess the credibility of the survey in the context of regression models.
The central objective is to develop an efficient algorithm that minimizes both computational cost and sample complexity (i.e., the number of additional samples required).



\textbf{Problem Formulation.} Typically, once a sampling-based study is designed, survey data is collected from an underlying population.
In line with the structure of most socio-economic surveys, we assume that the survey dataset $\surveyset$ consists of tabular numeric covariates and a scalar response variable. Specifically, each data point in $\surveyset$ is of the form $(\covariate, \response)$, where the covariates $\covariate \in \R^d$ and the response variable $\response \in \R$. Most of the time the dimension, that is $d$, is quite large.

We denote by $\D$ the distribution of the $(\covariate,\response)$ tuples of the whole population. 
If the dataset $\surveyset$ was obtained after perfect sampling techniques, i.e. by drawing independent samples from an unknown distribution $\D$, then one would call the survey data $\surveyset$ to be a credible representation of the population. 
But due to various limitations, the dataset $\surveyset$ collected
might be obtained by drawing samples from some other distribution 
$\sampD$. So the question about how credible is $\surveyset$ as a 
representation of the population boils down to understanding the 
distance between the two distributions $\D$ and $\sampD$. 
We will call $\D$ to be the \textit{true distribution} and 
$\sampD$ to be the \textit{sample distribution}.
Estimating the distance between two high-dimensional distributions is very inefficient, and hence, impractical~\citep{Canonne:Survey:ToC,CanonneTopicsDT2022}. This has motivated development of distance measures between datasets, such as Optimal Transport Dataset Distance~\citep{OTDT}, which are costly to compute in high dimensions.

In particular, we list the sample complexities of some of the most well-studied distributional distances when the distributions are defined over a $d$-dimensional space:
\begin{itemize}[leftmargin=*]
    \item \textbf{TV:} The problem of testing the TV distance of two distributions over a support of size $k$ require $\Theta(k/\log k)$ samples~\citep{pmlr-v178-canonne22a}. Given the distribution is over $\{0,1\}^d$, the sample complexity is $\Theta(2^d/d)$. If we have a continuous distributions over $[0,1]^d$ discretized with bin width $\varepsilon$, the sample complexity would be $\Theta(\varepsilon^d/(d\log(1/\varepsilon)))$.
    \item \textbf{Wasserstein:} For two bounded-moment distributions over a $d$-dimensional space, the Wasserstein distance requires $\Omega(\epsilon^{-d})$ samples for the empirical measure to converge to distance $\epsilon$~\citep{LeiWassersteinConcentration}. 
    \item \textbf{KL:} For two distributions over a $d$-dimensional bounded space, the minimax-optimal estimation of the KL divergence requires $\Omega(\varepsilon^{-d})$ samples~\citep{DBLP:journals/tit/ZhaoL20a}. 
\end{itemize}
In all the cases discussed above, to test closeness of distributions given sampling access to them, requires the number of samples to grow exponentially with the number of dimensions. In contrast, the number of samples-to-test required by our method is independent of dimension.

Samples collected from a survey are typically used for various data interpretation and deduction tasks, e.g. regression, classification etc.
In all these cases, one aims to find a model from a given model class, say $\regressionset$, that minimises a task-specific loss function. For example, for regression, we aim to find the regression function that minimise the square loss over the survey data.  If $\func{\loss}{\R^2}{\R}$ is the loss function, then the model learnt from the survey set $\surveyset$ is:
\begin{align*}
    \surveyf \triangleq \argmin_{f \in \regressionset} \frac{1}{m} \sum_{(\covariate,\response) \in \surveyset}\loss (f(\covariate),\response).
\end{align*}

To validate the credibility of a survey data, we propose to test whether the model $f_{\surveyset}$ derived from the survey data $\surveyset$ matches the model $\optimalf$, that would have been odtained if the dataset $\surveyset$ been a credible representation of the population $\D$. 
 \begin{align*}
    \optimalf \triangleq \argmin_{f \in \regressionset} \E_{(\covariate,\response)\sim \D} \loss(f(\covariate), \response).
\end{align*}
We will assume that we have access to a small sample set, called the validation dataset, obtained by drawing \textit{i.i.d.} samples from the true distribution $\D$. 


Depending on the problems, different metrics have been proposed to quantify the closeness of distributions~\citep{Gibbs2002_metrics}. Our goal is to validate the quality of the survey data $\surveyset$ by estimating the distance of $\surveyf$ from $\optimalf$.
 We use the distributional $\ell_2$ distance to quantify the closeness of regression models. 

\begin{definition}[Distributional $\ell_2$-Distance between Functions]\label{Definition: Statistical distance between functions}
      Let $f$ and $g$ be real-valued functions on $\R^{d}$, and $\sD$ be a distribution on $\R^{d}$. The distributional $\ell_2$-distance between $f$ and $g$ on $\sD$ is:

    $$\distltwo(f,g) \triangleq \sqrt{\E_{\covariate \sim \sD} \tbrac{(f(\covariate)-g(\covariate))^2}}$$
\end{definition}
Thus, our problem can be formulated as follows: Given a survey set $\surveyset$ (drawn according to some unknown distribution $\sampD$) and a model class $\regressionset$, we aim to sample a small number of new data points from the true distribution $\D$ and determine whether $\distF(\surveyf, \optimalf)$ lies within a specified acceptable threshold. Ideally, the number of new samples drawn from $\D$ should be very small and independent of the dimensionality of the ambient space.

\textbf{Related Works.} 
Our work lies at the intersection of \textit{distribution testing} and \textit{model validation}. Distribution identity testing—determining whether an unknown distribution matches a known one—has been widely studied ~\cite{Batu2000_equivalence,Paninski2008_identity,Valiant2017_identity,DBLP:conf/icalp/DiakonikolasGPP18}, with comprehensive surveys summarizing key results~\cite{CanonneTopicsDT2022,Canonne:Survey:ToC}. Recent efforts have focused on high-dimensional settings, where testing structured distributions such as Ising models or Bayesian networks poses significant challenges~\cite{pmlr-v65-daskalakis17a,Daskalakis/SODA/2018/TestingIsingModels,Canonne/PMLR/2017/TestingBayesianNetworks,NEURIPS2020_a8acc287,Bhattacharyya/ALT/2021/TestingProductDistributions}. 
However, these approaches often suffer from exponential sample complexity in the dimension $\d$~\cite{JMLR.2020/LowerBoundOnTestingIsingModels,Blanca/JMLR/2021/HardnessIdentity}.
In contrast, model validation has long been studied through statistical tests for evaluating model fit, especially in regression and parametric models~\cite{Snee01111977,Picard01091984,Holger/AOS/1998/ValidatingRegression,Staffa2021,Stute/AoS/1997/NonParametricModelChecksForRegression}. These approaches often rely on strong assumptions about the model or the data. 
Our work brings these two perspectives together aiming to develop scalable and principled methods for validating the credibility of high dimensional surveys  through the lens of regression models.


\textbf{Our Results.} 
In this work, we consider the class of regression models for the model-specific testing problem. 
We consider two common assumptions of regression models for our scenario -- exogenous noise in observation~\citep{rousseeuw2003robust,mohri2018} and boundedness of involved variables and the model~\citep{mohri2018,islr}.\footnote{Note that these two assumptions are not absolutely necessary for the proposed framework to function but to provide clean and rigorous theoretical analysis. We discuss further in Section~\ref{section:future work}.}
Exogeneity of noise ensures exact identifiability of the underlying model, i.e. we do not have unidentified covariates that influence the outcome. Boundedness is usually satisfied in our setting as the survey datasets always have finite entries and can be normalized.
\begin{assumption}[\textbf{Exogenous Noise}]\label{Assumption : Assumptions on the linear regression model}
For a regression model $\response = f(\covariate)+\regnoise$, we have: \\
(a) \emph{Homoskedasticity:} The noise $\regnoise$ has constant variance, i.e. $\Var[\regnoise\mid\covariate] = \regnoisevariance$, \\
(b) \emph{Non-correlation:} The noise $\regnoise$ is uncorrelated with $\covariate \in \R^\d$ and independent across observations.
\end{assumption}
\begin{assumption}[\textbf{Boundedness}]\label{Assumption: Bounded Variables} 
     We assume that the response variable satisfy $\abs{\response}\leq 1$, the covariates satisfy $\norm{\covariate}_\infty \leq 1$, and {$f(\covariate) \leq 1$.}
\end{assumption}

Given this context, we elaborate the main contributions of this paper:

\noindent1. \textbf{Task-Specific Credibility Testing:}  We propose the framework of task-specific credibility testing of survey that checks whether it leads to valid inference while used with ML models. Specifically, we focus on regression models -- linear with $\ell_1$ and $\ell_2$ regularizers, and kernel with  $\ell_2$ regularizers. This is a deviation from the classic distribution testing frameworks that check for some divergence (e.g. TV, KL, Wasserstein) between two data distributions. But these frameworks require exponential number of samples with respect to the dimension of data. This is infeasible for a survey setting. Thus, we propose a new data-distribution specific metric, called the \textit{Functional Distance of Distributions (\distname)}, between two regression model, and leverage it to test closeness of two data distributions through the lens of regression. 

\noindent2. \textbf{Generic Algorithm for Model-Specific Testing for Regression Models:} We propose \SurVerif~to test whether a regression model learned from a given survey data $\surveyset$ is close to a model learned using independent and identically distributed (\textit{i.i.d.}) samples collected from an underlying distribution. \SurVerif~does this by checking whether the loss of the survey-based model and the \textit{i.i.d.} model match up to pre-computed threshold. We prove that \SurVerif~is correct with high probability up to a user-defined tolerance gap. 
We show that the worst-case sample complexity\footnote{Sample complexity is the number of \textit{sample-to-test} the \SurVerif~needs from true distribution $\D$.} of \SurVerif~to conduct a correct test is independent of the dimension and fixed across regression models. 
Additionally, if the model is very far in the {\distname} metric, \SurVerif~detects it earlier with less samples. 
Finally, we numerically verify the correctness and sample complexity of \SurVerif~across datasets.

To conduct our theoretical analysis, we propose a new two-sided bound on generalization error of a regression model, which is of independent interest for statistical learning.


\paragraph{Organization of the paper:} 
Section~\ref{Section : Linear Regression} introduces the preliminaries. Section~\ref{section: functional distance} discusses the new metric. Section~\ref{section: SurVerify} presents our main algorithm, \SurVerif{}, with theoretical guarantees. Proofs appear in the Appendix. Section~\ref{section:experiment} reports experimental results. 


\section{Preliminaries: Regression Models and Rademacher Complexity}\label{Section : Linear Regression}

The survey set is denoted as $\surveyset$, and its size as $\sizeS = \abs{\surveyset}$. 
We denote $\xdomain$, and $\yrange$ to be the input and output spaces, respectively. $\hypothesisset$ denotes a hypothesis sets consisting of hypothesis $\func{\hypothesis}{\xdomain}{\yrange}$. Similarly, $\regressionset$ denotes the set of regression functions $\func{\regression}{\xdomain}{\yrange}$, and the coefficient associated with the regression functions are denoted $\coeff$. $\inn{\cdot}{\cdot}$ denotes  inner product, and $\norm{\cdot}_p$ denotes the $\ell_p$ norm.

\paragraph{A Primer on Regression: Linear and Kernel.} Performing regression on survey data to fit reasonable models over the population is central to a wide variety of analysis tasks~\citep{Charvat2015_LR,Pan2017_LR,Meerwijk2017_LR}. Often, the observations collected to construct a survey dataset are the result of a complex sampling design reflecting the need to collect data as efficiently as possible within cost constraints.


Broadly, the problem of regression is as follows: given an input space $\xdomain\subseteq\R^\d$, an output range $\yrange \subseteq \R$, a distribution over $\xdomain\times\yrange$, a hypothesis set $\func{\hypothesisset}{\xdomain}{\yrange}$, and a loss function $\func{\loss}{\yrange\times\yrange}{\R}$, output a hypothesis $h \in \hypothesisset$ that minimizes loss w.r.t. the distribution over $\xdomain\times\yrange$. Here, we consider the regression model with additive noise $\regnoise$. That is,
$
\response = f(\covariate) + \regnoise
$.

In this work, we consider three widely used hypothesis classes for the regression problem. 
First, we consider  linear regression that tries to fit a linear model between the response and the covariates, i.e.
\begin{equation*}
\response = \inn{\truecoeff}{\covariate}+\regnoise, \text {where } \truecoeff,\covariate\in\R^\d,\response,\regnoise\in\R
\end{equation*}

We consider both the cases of $\ell_1$ and $\ell_2$-norm bounded coefficients for the linear regression model, known as $\lasso$ and $\ridge$ regression respectively. 
These are also called the bounded weight hypothesis classes  $\boundedlinfunc_p = \{\covariate \rightarrow \inn{\coeff}{\covariate}:\norm{\coeff}_p \leq \linweight\}.$
Henceforward, we use these two terms interchangeably. We denote the hypothesis sets containing $\ell_1$ and $\ell_2$ bounded linear regressions as $\lassoset$, and $\ridgeset$, respectively.

We also consider the $\kerna$ Regression model, where we
associate with the input space $\xdomain$ a PDS (Positive Semidefinite Symmetric) kernel $\func{\kernel}{\xdomain\times\xdomain}{\R}$ that implicitly defines an associated function $\func{\kernelfunction}{\xdomain}{\hilbertspace}$ such that:
$\kernel\fbrac{\covariate,\covariate'} = \inn{\kernelfunction(\covariate)}{\kernelfunction(\covariate')}$.
The regression model is a linear model on this Hilbert space $\hilbertspace$ with the underlying coefficients $\truecoeff \in \hilbertspace$, and the model is:
\begin{align*}
    \response = \inn{\truecoeff}{\kernelfunction(\covariate)} + \regnoise \text{ where }\truecoeff,\kernelfunction(\covariate)\in\hilbertspace,\response,\regnoise\in\R
\end{align*}
In this case, we consider the hypothesis class consisting of coefficients $\coeff$ with bounded $\hilbertspace$-norm. We denote the hypothesis classes containing the kernel as $\kernelset$. For all the regression models, we consider the loss function $\loss$ to be the squared error loss function defined as $\loss(\response,\response') \triangleq (\response - \response')^2$.





\paragraph{Rademacher Complexity.} The Rademacher complexity of a function class $\losshypothesisset$ plays a crucial role in the generalization bounds for several learning models~\cite{mohri2018}, and also in our analysis. The empirical Rademacher complexity is measured w.r.t. a particular set of samples $\surveyset$.

\begin{definition}[\textbf{Empirical Rademacher Complexity}]\label{Definition: Empirical Rademacher Complexity}
    Given a family of functions $\losshypothesisset$ containing functions $\func{\losshypothesis}{\datadomain}{\tbrac{0,\lossbound}}$ and $\surveyset = \fbrac{\datapoint_1,\ldots,\datapoint_\sizeS}$ a fixed sample with elements in $\datadomain$. Then, the empirical Rademacher complexity $\empradcomp{\surveyset}{\losshypothesisset}$ of $\losshypothesisset$ w.r.t. $\surveyset$ is
    \[
    \empradcomp{\surveyset}{\losshypothesisset} \triangleq \E_{\mathbf{r}}\tbrac{\sup_{\losshypothesis \in \losshypothesisset} \frac{1}{\sizeS}\sum_{i \in \sizeS} \rademacher_i \losshypothesis(\datapoint_i)},
    \]
    where 
    $r_{i}$'s are i.i.d Rademacher random variables taking value uniformly in $\sbrac{-1,+1}$.
\end{definition}



\section{Functional Distance of Distributions (\distname): A Novel Metric}\label{section: functional distance}



We define the model-specific distance between distributions that quantifies the distance between distributions w.r.t. a model class $\regressionset$ and a true distribution $\D$.

\begin{definition}[$\distname_{\D}^\regressionset(\distribution_1,\distribution_2)$]
      Given a true distribution $\D$, a model class $\regressionset$, and an associated loss function $\loss_\regressionset$, let $\regression_{\distribution_1}$, and $\regression_{\distribution_2}$ be the optimal models in $\regressionset$ for $\distribution_1$, and $\distribution_2$, respectively. We define the model-specific distance w.r.t. the true distribution $\D$ as:
      \begin{align*}
          \distname_{\D}^\regressionset(\distribution_1,\distribution_2) = \distF\fbrac{\regression_{\distribution_1},\regression_{\distribution_2}}\,.
      \end{align*}
\end{definition}

Given a true distribution $\D$, the model specific testing transforms the problem of testing closeness of distributions to testing closeness of functions over a given true distribution. Given a hypothesis set $\hypothesisset$, and a loss function $\loss$, it associates with each distribution $\distribution$ a function $\hypothesis_\distribution$ as
$
\hypothesis_\distribution = \argmin_{\hypothesis \in \hypothesisset} \E_{\covariate,\response\sim\distribution}\loss\fbrac{\hypothesis(\covariate),\response}
$.

Consequently, given a set of distributions $\distset$, we can define the set of hypotheses associated with them as $
\hypothesisset^\distset = \sbrac{\hypothesis_\distribution\in\hypothesisset\mid\distribution\in\distset}
$.

By a standard fact of $\Lspace^p$ spaces~\cite{SteinShakarchi+2012}, if the functions $\hypothesis \in \hypothesisset^\distset$ has bounded second moment w.r.t. $\D$, i.e. $\E_{(\covariate, \response)\sim \D}\tbrac{\hypothesis^2(\covariate)} < \infty$, then the set $\fbrac{\hypothesisset^\distset,\distF{}{}}$ constitutes a $\Lspace^2$-space. 
If we consider the equivalence relation $\hypothesis_1\sim\hypothesis_2$, i.e., if and if  $\distF(\hypothesis_1,\hypothesis_2) = 0$, $\distF{}{}$ defines a metric on the resulting partition. Correspondingly, $\distname_{\D}^\regressionset$ induces a metric on the partition of distributions $\distset$ induced by the equivalence relation $\distribution_1 \sim \distribution_2$ if and only if $\distF\fbrac{\regression_{\distribution_1},\regression_{\distribution_2}} = 0$.

It is important to note that the \distname{} metric can be zero even when the distributions $\distribution_1$ and $\distribution_2$ differ significantly. Therefore, when the goal is to assess whether two distributions are equivalent with respect to a specific task, \distname{} serves as an appropriate measure. For regression models that satisfy the exogenous noise assumption (Assumption~\ref{Assumption : Assumptions on the linear regression model}) under the squared loss, we establish the following relationship between the loss and \distname{}.




\begin{restatable}[\distname-variance Decomposition of Loss]{lemma}{distanceLemma}\label{lemma:lossvariancebreakdown}
  \label{lem:dist} If the model class $\regressionset$ satisfies Assumption~\ref{Assumption : Assumptions on the linear regression model}, then $$\E_{(\covariate, \response)\sim \D}\tbrac{\fbrac{f_\surveyset(\covariate)-\response}^2} = \red{(}\distname_{\D}^\regressionset(\sampD,\D)\red{)^2} + \regnoisevariance.$$
  
\end{restatable}

The above lemma can be intuitively viewed as a decomposition result, akin to the classic bias-variance breakdown of estimation error \cite{wasserman2004all}. It states that the expected loss of a model learned from the survey, evaluated with respect to the true distribution, can be decomposed into two components: the approximation error (i.e., how far the learned model is from the optimal one) and the intrinsic noise (i.e., the error incurred even by the best possible model).





\section{\SurVerif: Testing Credibility with Regression and Fixed Confidence}\label{section: SurVerify}
We first describe the algorithm design and then establish its efficiency in terms of sample complexity. 
In order to prove this result, we propose a two-sided generalization bound for regression and also a lower bound on methods reconstructing complete model to test dataset distances.

\subsection{Dimension Agnostic Algorithm Design with Early Stopping}

We now present our algorithmic framework, \SurVerif{}, which verifies whether a regression model learned from a survey sample $\surveyset$ is close to the true optimal model in $\ell_2$-distance (Definition~\ref{Definition: Statistical distance between functions}). The algorithm performs this testing using a small number of samples drawn from the true distribution $\D$. We refer to them as \textit{sample-to-test}.

\begin{restatable}{algorithm}{surverify}
\caption{\SurVerif($\surveyset \subset \mathbb{R}^{(d+1)},\D, \epsilon, \confidence, \regressionset $)}\label{alg:SurVerif}
\begin{algorithmic}[1]
\Require $\abs{\surveyset} \geq \sbound(\regressionset)$

\State Initialize $\sizeS \gets |\surveyset|$, $\surveyset_{\D} \gets \emptyset, \totalsampD \gets \ceil{\frac{2}{(1.9\tol)^2}\log\fbrac{\frac{3}{\confidence}}}$ \label{SurVerif Line 1}

\State $f_{\surveyset} \gets \argmin_{f \in \regressionset} \frac{1}{\sizeS} \sum_{(\covariate,\response) \in \surveyset} (f(\covariate) - \response)^2$ \label{SurVerif Line 2}

\State $\hat{L}_S \gets \frac{1}{\sizeS}\sum_{(\covariate,\response) \in \surveyset} (f_{\surveyset}(\covariate) - \response)^2$ \label{SurVerif Line 3}

\State $\estimatevarD \gets 0, \sampleD \gets 0$ \label{SurVerif Line 4}
\While{$\sampleD < \totalsampD$} \label{SurVerif Line 5}
    \State $\surveyset_{\D} \gets \surveyset_{\D} \cup \{(\covariate_i, \response_i)\}$, where $(\covariate_i, \response_i) \sim \D$ \label{SurVerif Line 6}
    \State $\estimatevarD \gets \estimatevarD + (f_\surveyset(\covariate_i) - \response_i)^2$ \label{SurVerif Line 7}
    \If{$\estimatevarD - t\hat{L}_S > 1.1 \sampleD\tol + \sqrt{2\sampleD\log\fbrac{\frac{3\totalsampD}{\confidence}}}$} \label{SurVerif Line 9}
    \State \Return REJECT \label{SurVerif Line 10}
    \EndIf \label{SurVerif Line 11}
    \State $\sampleD \gets \sampleD + 1$ \label{SurVerif Line 8}
\EndWhile \label{SurVerif Line 12}
\If{$\estimatevarD - \totalsampD\hat{L}_S \leq 3\totalsampD \tol$} \label{SurVerif Line 13}
\State \Return ACCEPT \label{SurVerif Line 14}
\Else \label{SurVerif Line 15}
\State \Return REJECT \label{SurVerif Line 16}
\EndIf \label{SurVerif Line 17}
\end{algorithmic}
\end{restatable}

We begin with an overview of our algorithmic framework, \SurVerif{}, before presenting its formal correctness guarantee. The core idea behind \SurVerif{} is to assess the credibility of a survey sample $\surveyset$ through a two-phase procedure. In the first phase (Lines~\ref{SurVerif Line 2} and \ref{SurVerif Line 3}), the algorithm fits a regression model $f_{\surveyset}$ using the survey data. In the second phase (Lines~\ref{SurVerif Line 5} to \ref{SurVerif Line 17}), it evaluates the reliability of $f_{\surveyset}$ by estimating its expected loss under the true distribution $\D$, using a small number of i.i.d. samples-to-test. Specifically, it computes an additive estimate $\estimatevarD$ of the expected loss of $f_{\surveyset}$ on data from $\D$. The algorithm then compares $\estimatevarD$ against a fixed threshold: if the estimated loss is low enough (Line~\ref{SurVerif Line 13} and onward), it outputs ACCEPT; otherwise, it outputs REJECT.

To be more sample-efficient, \SurVerif{} also incorporates an early rejection criterion (Line~\ref{SurVerif Line 9}) to terminate the evaluation of $f_{\surveyset}$ quickly when it incurs a large loss on the sample-to-test, i.e., when the loss is deviating enough to be detected with only a few samples. Notably, the total number of samples-to-test required from $\D$, denoted $\totalsampD$, is $O\left(\frac{1}{\tol^2} \log\left(\frac{1}{\confidence}\right)\right)$, and is independent of the data dimension. This sample efficiency makes \SurVerif{} well-suited for high-dimensional settings where direct access to the true distribution is limited, and also in the settings where collecting samples is costly (e.g. medical data). 

\subsection{Theoretical Analysis: Correctness, Sample Complexity, and Sufficient Size of Survey}

The following theorem is the main structural result of this work. It shows that the validity of a model learned from survey data can be efficiently certified using only a small number of i.i.d. samples-to-test from the true distribution. This is especially useful when survey data is abundant but access to the true distribution is limited (e.g. medical data, socioeconomic data). By leveraging the framework of functional distance of distributions (defined in Section~\ref{section: functional distance}), \SurVerif{} reliably distinguishes between two datasets with high confidence and low sample complexity.

\begin{restatable}[\textbf{Correctness of \SurVerif{} and Sample Complexity}]{theorem}{surverifyCorrectness}
    \label{Theorem: Tester Accepts or Rejects}
     Given a survey sample $\surveyset$ (drawn from an unknown distribution $\sampD$), a model class $\regressionset$ and i.i.d.~sampling access to the true distribution $\D$ then for any  $\tol$ and $\confidence \in (0,1)$, if the size of $\surveyset$ is large enough (Table~\ref{table: sample complexity}) then
\begin{itemize}
    \item[1.] If $\red{(}\distname_{\D}^{\regressionset}(\sampD, \D)\red{)^2} \leq \tol$, then \SurVerif{} outputs ACCEPT with probability $1 -\confidence$. \label{subTheorem: close statement}   
    \item[2.] If $\red{(}\distname_{\D}^{\regressionset}(\sampD, \D)\red{)^2} > 5\tol$, then \SurVerif{} outputs REJECT with probability $1 -\confidence$.\label{subTheorem: far statement}
\end{itemize}
Also, \SurVerif{} requires at most $\ceil{\frac{2}{(1.9\tol)^2}\log\fbrac{\frac{3}{\confidence}}}$ samples from $\D$ for validation.
\end{restatable}

    
\textbf{Discussions:} 1. \textit{Dimension Agnostic Tester:} One of the interesting aspect of the above theorem is the fact that the sample complexity of \SurVerif{} is independent of dimension. This efficiency stems from the fact that the algorithm focuses on verifying the credibility of the survey data rather than reconstructing the underlying regression model.

2. \textit{Relaxing Purity of Samples:}  At first glance, Theorem~\ref{Theorem: Tester Accepts or Rejects} may seem limited in practical applicability, as it assumes access to the true distribution $\D$. However, in real-world settings, we typically have access only to a distribution $\distribution'$ that is close to $\D$, for instance in total variation distance. Fortunately, since the sample complexity of \SurVerif{} is $O(1)$ for fixed $\tol$ and $\confidence$, the algorithm remains effective in this approximate setting. By appropriately adjusting the tolerance and confidence parameters to account for the discrepancy between $\D$ and $\distribution'$, we can still guarantee the correctness of the testing procedure. This robustness follows directly from the Data Processing Inequality~\cite{Polyanskiy_Wu_2025}.

3. \textit{Dealing with Regression Models on a Subset of Dimensions:} Oftentimes, broad survey data is used for various downstream tasks involving projections onto a small number of dimensions. However, the \distname{} metric is not robust to arbitrary projections—closeness between entire datasets does not necessarily imply closeness under such projections. In these cases, the only reliable approach is to run \SurVerif{} on the projected dimensions. Fortunately, the same sample from $\D$ can be reused across multiple projection-based checks.

4. \textit{Fixing $\approxerror$, and $\confidence$ in practice:} The choice of $\delta$ is generally taken within the range of $[0.01,0.1]$ in practice. Although due to the fact that the dependence of sample complexity on $\delta$ is logarithmic, choosing a lower value does not impact the sample complexity much. The tolerance parameter $\varepsilon$ should be chosen according to the confidence required w.r.t. the underlying noise $\regnoise$. Given the fact that testing w.r.t. a lower $\varepsilon$ does not cause an increase in sample complexities in practice, one strategy may be to test it with lower value of $\varepsilon$ and obtain a (constant factor) estimate to the $\distname{}$ using \SurVerif{}. If there is a fixed number of samples to test with, the strategy should be to fix the $\varepsilon$ level theoretically attainable according to Theorem~\ref{Theorem: Tester Accepts or Rejects}.



\paragraph{Requirement: Sufficient Size of the Survey Data.}
We show the following two-sided generalization bound of a general hypothesis class using the empirical Rademacher complexity.
\begin{restatable}[\textbf{Generic Two-sided Generalization Bound}]{theorem}{twoSidedGeneralizationBound}
   \label{theorem: Regression Generalization Bound}
   Given a hypothesis set $\hypothesisset$ containing functions $\func{\hypothesis}{\xdomain}{\yrange}$, and a $\losslipschitz$-lipschitz\footnote{As per the standard nomenclature, a 
    loss function $\sL: \responserange \times \responserange \rightarrow \R$ is called $\losslipschitz$-Lipschitz if for any fixed $y \in \responserange$ and $\response_1,  \response_2 \in \responserange$, we have $|\sL(\response_1,y) - \sL(\response_2,y)| \leq \losslipschitz|\response_1 - \response_2|$.} loss function $\func{\loss}{\yrange\times\yrange}{\tbrac{0,\lossbound}}$. Let $\surveyset$ be a sample set of size $\sizeS \geq 1$ drawn as i.i.d. samples from the distribution $\distribution$, then we have with probability at least $1 - \confidence$:
    \begin{align*}
\abs{\E_{\fbrac{\covariate,\response}\sim\distribution}\tbrac{\loss\fbrac{\hypothesis(\covariate),\response}} - \frac{1}{\n}\sum_{(\covariate, \response)\in \surveyset} \loss\fbrac{\hypothesis(\covariate),\response}} \leq 2\losslipschitz\empradcomp{\surveyset}{\hypothesisset} + 3\lossbound\sqrt{\frac{\log\frac{4}{\confidence}}{2\n}}
    \end{align*}
\end{restatable}
Note an upper bound in the generalization bound can be found in the following textbook~\cite{mohri2018}. 
We extend this to a two-sided bound controlling both under and overestimation. This is particularly important 
since we aim to design a tolerant tester.

Note that computuing empirical Rademacher complexity $\empradcomp{\surveyset}{\hypothesisset}$ is 
known to be computationally hard for general hypothesis classes~\cite{froese2024trainingneuralnetworksnphard, manurangsi2018computationalcomplexitytrainingrelus}. However, for a bounded weight linear and kernel basel class, the $\empradcomp{\surveyset}{\hypothesisset}$ admits tight analytical bounds (see \cite{awasthi2020rademacher, mohri2018}). We use this fact together with Theorem~\ref{theorem: Regression Generalization Bound} to bound the sample size needed for estimating the noise variance.




The following result gives the size of the survey data needed for 
estimating $f_\surveyset$ for $\lasso$, $\ridge$ and $\kerna$ hypothesis classes.

\begin{restatable}[\textbf{Minimum Survey Size for Learning Noise Variance}]{lemma}{sampleComplexityNoiseVariance}\label{lemma: sample complexity noise variance}
    Given a survey $\surveyset$ of size $\sizeS$ which is sufficiently large for their respective linear hypothesis classes (see Table \ref{table: sample complexity}). If Assumptions~\ref{Assumption : Assumptions on the linear regression model} and \ref{Assumption: Bounded Variables} hold, then 
    with probability at least $1 - {\confidence}/{3}$ we have 
    \[
        \abs{\regnoisevariance - \frac{1}{\sizeS}\sum_{(\covariate,\response) \in \surveyset} (f_{\surveyset}(\covariate) - \response)^2} \leq \frac{\tol}{10},
    \]
    where $f_{\surveyset} \triangleq \argmin_{f \in \boundedlinfunc} \frac{1}{\sizeS} \sum_{(\covariate,\response) \in \surveyset} (f(\covariate) - \response)^2$. Note that 
    Table~\ref{table: sample complexity} gives the sufficient survey data size  
    from distribution for $\sD_\surveyset$ for $\lasso$, $\ridge$ and $\kerna$
    hypothesis classes.  \vspace*{-.6em}
\end{restatable}
    \begin{table}[ht!]
      \caption{Sufficient Size of Survey Data 
      }
      \label{table: sample complexity}
      \centering
      \begin{tabular}{llll}
        \toprule
        Hypothesis Class($\boundedlinfunc$) & $\boundedlinfunc_1$(\lasso)     & $\boundedlinfunc_2$(\ridge)     & $ \kernelset
        $(\kerna) \\
        \midrule
        Size of $\surveyset$ ($\sbound(\regressionset)$) & $\Omega\fbrac{\frac{\log(d)} {\tol^2}}$ & $\Omega\fbrac{\frac{d}{\tol^2}}$  & $\Omega\fbrac{\frac{\kernelnormbound^2}{\tol^2}}$\tablefootnote{$\kernelnormbound^2$ is  the upper bound of        $|\kernel\fbrac{\covariate,\covariate'}|$.}    \\
        \bottomrule
      \end{tabular}
    \end{table}

\begin{remark}[\textbf{$\sparsity$-Sparse Linear Regression}]\label{remark: s-sparse regression}
    For the hypothesis class $\boundedlinfunc_1$($\lasso$), one might be interested in $\sparsity$-sparse linear regression, In that case we consider the coefficient vector $\coeff$ to be
    $\sparsity$-sparse  and the hypothesis class is defined by  $\{\covariate \rightarrow \inn{\coeff}{\covariate}:\norm{\coeff}_1 \leq \linweight, \norm{\coeff}_0 \leq \sparsity\}$. Given survey data of size $\sizeS = \Omega\fbrac{\frac{\log(\sparsity)}{\tol^2}}$ from the distribution. If Assumptions \ref{Assumption : Assumptions on the linear regression model} and \ref{Assumption: Bounded Variables} hold, then with probability at least $1 - \confidence$, we have $\abs{\regnoisevariance - \frac{1}{\sizeS}\sum_{(\covariate, \response) \in \surveyset} (f_\surveyset(\covariate)-\response)^2} \leq \frac{\tol}{10}$. \color{red}
    \color{black}
\end{remark}

\textbf{Discussion: Relation to Out-Of-Distribution (OOD) Generalization.} The OOD generalization literature assumes an intrinsic model can be learned across distributions, i.e. the performance of the learned hypothesis generalizes well to OOD data (Assumptions A--D in~\citep{liu2023OODgeneralizationsurvey}).  Our mechanism, on the other hand, works on the case where sampling from a different distribution results in a different model being learned. In other words, if there is an intrinsic model that can be learned across distributions, the distance $\distname_{\D}^{\regressionset}(\sampD, \D)$ for the model class $\regressionset$ would be $0$ for all distributions $\sampD$ and $\D$. However, if that is not the case, we would efficiently detect whether the model learned from the survey distribution $\sampD$ generalizes well to the true distribution $\D$.

\subsection{Lower Bound on Sample Complexity: Advantage of Not Reconstructing the Model}

{\SurVerif} tests the model-specific credibility of a given sample survey without reconstructing the model itself. The fact that we don't reconstruct the model helps us to ensure that the sample complexity is independent of the dimension.
The following lemma proves that the number of samples that any algorithm that reconstructs the model to estimate model-specific distance needs grows linearly with dimension.

\begin{restatable}[\textbf{Lower Bound on Testing with Model Reconstruction}]{lemma}{LBReconst}\label{Lemma: Lower Bound for Reconstruction}
  Under Assumption~\ref{Assumption: Bounded Variables}, and when $\mineigen\fbrac{\mathrm{Cov}\fbrac{\covariate}} \geq \mineigen{}$, any algorithm that reconstructs the model to estimate the distance $\distname_{\D}^{\regressionset}(\sampD, \D))$ within $\approxerror$ additive error  must make $\bigomega{\frac{\d\mineigen\regnoisevariance}{\approxerror^2}}$ queries.
\end{restatable}

Furthermore, if $\sampD$ and $\D$ are two distributions such that their respective loss distributions are subgaussian distributions with same variance but the means differ by $\approxerror$, then $\distname_{\D}(\sampD, \D) = \approxerror$ (by Lemma~\ref{lem:dist}). Since distinguishing between the two such subgaussian distributions 
requires $\Omega(1/\approxerror^2)$ we observe that the 
sample complexity of {\SurVerif} is tight in terms of dependence on $\approxerror$.

\section{Experimental Analysis}\label{section:experiment}
In this section, we empirically verify whether our tester  \SurVerif{} performs  as per the theoretical analysis. 
In particular, we are interested in the following research questions:

\textbf{RQ1.} Does \SurVerif{} yield \textit{accept} when the survey data $\surveyset$ is close to being a credible dataset with respect to the model class, and likewise, does \SurVerif{} indeed reject when $\surveyset$ is far from being credible? 
Specifically, how does the acceptance rate of  \SurVerif{} change as the 
the distance between the survey set $\surveyset$ and the true distribution $\valD$, and the tolerance parameter change?

\textbf{RQ2.} How many \textit{i.i.d.} samples-to-test from the true distribution $\valD$ does \SurVerif{} require to certify if the survey data $\surveyset$ is credible? While the theoretical guarantee is for the worst-case runtime of \SurVerif{}, we would like to check if \SurVerif{} can reject a far from credible survey data $\surveyset$ with much less number of \textit{sample-to-test}.





\begin{figure*}[t!]
\centering%
\begin{minipage}{0.48\textwidth}
\centering
\includegraphics[scale=.5, width=.8\linewidth]{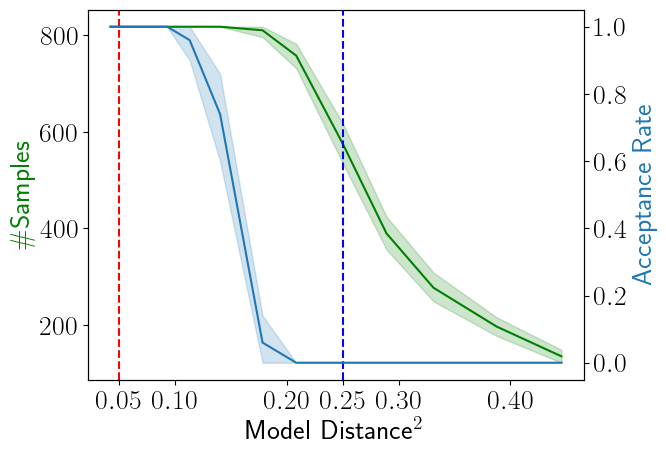}
\caption{\small{Acceptance rate of \SurVerif{} w.r.t. model class $\ridgeset$ on Synthetic Data vs. change in $\mu$ (over 50 runs) for $\delta = 0.1$ and $\epsilon = 0.05$.}}\label{Fig: SynthRidge}
\end{minipage}\hspace*{1em}
\begin{minipage}{0.48\textwidth}
\centering
\includegraphics[scale=.5, width=.8\linewidth]{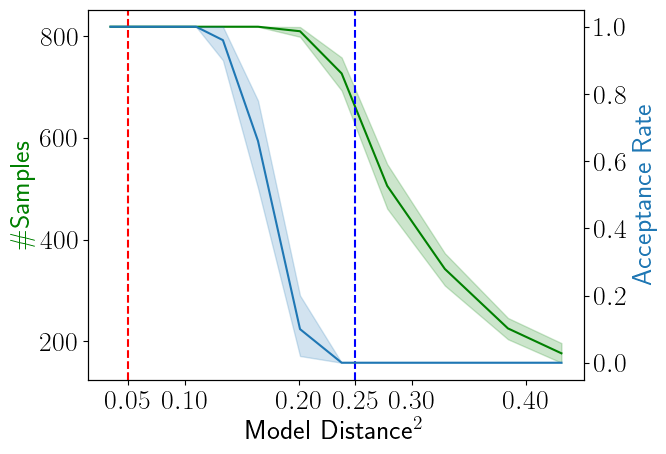}
\caption{\small{Acceptance rate of \SurVerif{} w.r.t. model class $\lassoset$ on Synthetic Data vs. change in $\mu$ (over 50 runs) for $\delta = 0.1$ and $\epsilon = 0.05$.}}\label{Fig: SynthLASSO}
\end{minipage}
\\\vspace*{1em}
\begin{minipage}{0.48\textwidth}\centering
\includegraphics[scale=.5, width=.8\linewidth]
{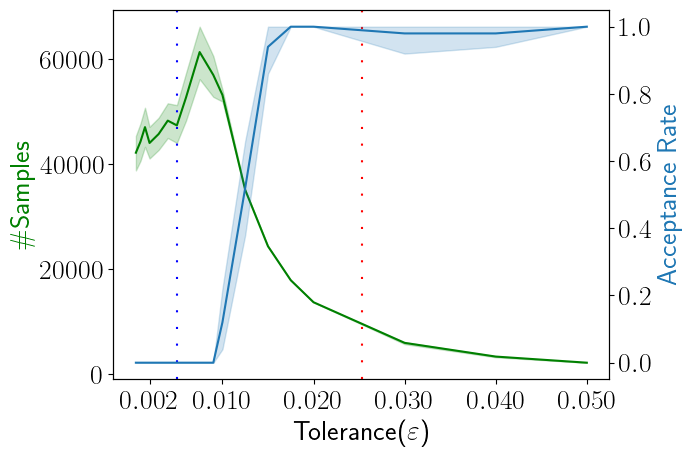}
\caption{\small{Acceptance rate of \SurVerif{} w.r.t. model class $\ridgeset$ on \acsincome~ (over 50 runs) for $\delta = 0.1$ and varying range of $\epsilon$.}}\label{Fig: ACSRidge}
\end{minipage}\hfill
\begin{minipage}{0.48\textwidth}
\centering
\includegraphics[scale=.5, width=.8\linewidth]{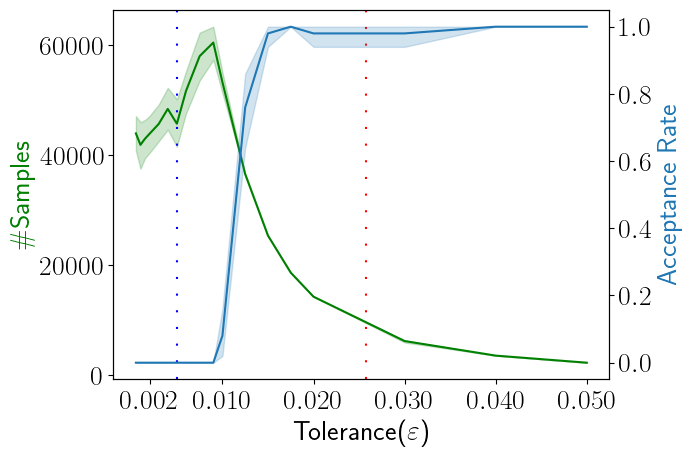}
\caption{\small{Acceptance rate of \SurVerif{} w.r.t. model class $\lassoset$ on \acsincome~ (over 50 runs) for $\delta = 0.1$ and varying range of $\epsilon$.}}\label{Fig: ACSLasso}
\end{minipage}
\end{figure*}

\textbf{Experimental Setup.} We implement all the algorithms in Python 3.10 and use \texttt{LinearRegression} from \texttt{scikit-learn} to learn $\surveyf$. We run our simulations on Google Collaboratory with 2 Intel(R) Xeon(R) CPU @ 2.20GHz, 12.7GB RAM, and 107.7GB Disk Space. 

\textbf{Setup 1: Synthetic.} We generate a synthetic dataset, where each coordinate of each $\covariate$ is generated from $\normal(0,1)$, and $\regnoise$ is generated from $\normal(0,0.1)$. For $\sampD$, we generate $\coeff_{\surveyset} \in \R^{50}$ such that each coordinate is from $\normal(0,0.01)$. The size of our set $\surveyset$ thus obtained is 100,000. For $\valD$, we generate the coefficients $\coeff^*$ with each coordinate being generated from $\normal(\mu,0.01)$ with $\mu$ taking values from $0$ to $3$ at intervals of $0.1$. As the value of $\mu$ increases the model distance between $f_\surveyset$ and $\optimalf$ increases. 

\textbf{Setup 2: \acsincome.} As a real-world dataset, we consider the normalized \acsincome{} dataset, which exhibits well-known fairness issues between Gender and Racial groups~\citep{DBLP:conf/nips/DingHMS21}.
We chose $\surveyset$ to be generated through sampling from the subpopulation with the parameter Sex set to $2$ (\texttt{Female}), and the distribution  $\valD$ to be the subpopulation with Sex set to $1$ (\texttt{Male}). An important observation regarding this dataset is that the dataset does not satisfy the homoskedastcity assumption (Assumption~\ref{Assumption : Assumptions on the linear regression model}). In particular, over 50 trials, the correlation coefficient between the response variable $\response$, and the residuals $\response - f(\covariate)$ w.r.t. the model $f$ obtained is $0.88$.


\textbf{Results and Observations.} The findings from the experimental results on both the synthetic and the real-world data corroborate our theoretical results. The details are as follows:

\textbf{Findings related to RQ1:}  We run $\SurVerif{}$ on each of the synthetic datasets and the \acsincome{} dataset 50 times and record the average performance and $95$ percentile around it. 

\noindent1. \textbf{Acceptance Rate on Synthetic.}
In Figure~\ref{Fig: SynthRidge}~and~\ref{Fig: SynthLASSO}, the {\color{Cerulean} \textbf{BLUE}} curve indicates the acceptance rate of \SurVerif{} on synthetic datasets described above w.r.t. $\ridgeset$ ($\ridge$) and $\lassoset$ ($\lasso$), respectively. For both the model classes of $\lassoset$ and $\ridgeset$,
\SurVerif{} exhibits similar behavior. It \textit{starts with accepting all models} when the difference
of the coefficients, and correspondingly, \textit{the model distance is small}. As the difference between the 
coefficients, and correspondingly \textit{the model distance increase, \SurVerif{} starts rejecting with increasing probability}, and rejects all the models generated with $\mu \geq 0.9, \distname \geq 0.20$ (resp. $\mu \geq 1, \distname \geq 0.21$) for model class $\ridgeset$ (resp. $\lassoset$).  The {\color{red} red} and {\color{blue} blue} dashed vertical lines indicate the value of $\tol$ and $5\tol$ respectively. Hence, when the model-distance lies to the right of the {\color{blue} blue} line, $\SurVerif{}$ is expected to reject, whereas values to the left of the {\color{red} red} line are expected to be accepted validating our theoretical results.

\noindent2. \textbf{Acceptance Rate on \acsincome{}:}
In Figure~\ref{Fig: ACSRidge} and \ref{Fig: ACSLasso}, the {\color{Cerulean} \textbf{BLUE}} line indicates the acceptance rate of \SurVerif{}
on \acsincome{} w.r.t. $\ridgeset$ and $\lassoset$, respectively. 
We run \SurVerif{} with varying tolerance parameter $\approxerror$. 
\SurVerif{} always rejects for $\approxerror$ less than $0.01$, and accepts for higher values.
The {\color{red} red} and {\color{blue} blue} dotted vertical lines indicate the value of $\distname\fbrac{\sampD,\D} \approx 0.02$ and $\distname\fbrac{\sampD,\D}/5$ respectively. Hence, as expected, we observe that for values of $\tol$ to the right of the {\color{red} red} line, $\SurVerif{}$ is accepts more often, while for values to the left of the {\color{blue} blue} line, $\SurVerif{}$ rejects. This further indicates that the $\distname\fbrac{\sampD,\D}$ between male and female subpopulations of \acsincome{} is at least $0.02$ with probability $0.9$.

\textbf{Findings related to RQ2: Sample Complexity.}
In Figure~\ref{Fig: SynthRidge}, \ref{Fig: SynthLASSO}, \ref{Fig: ACSRidge} and \ref{Fig: ACSLasso}, the {\color{Green} \textbf{GREEN}} curve demonstrates \#samples-to-set from $\sampD$ that $\SurVerif{}$ needed. As expected, in Figure~\ref{Fig: SynthRidge} and \ref{Fig: SynthLASSO}, i.e., while running on the synthetic dataset, as long as \SurVerif{} accepts \#samples-to-set are as per the worst-case complexity. But as the distance increases and the acceptance rate of \SurVerif{} decreases, the number of \#samples-to-set needed to reject also decreases.  
For both the model classes $\ridgeset$, and $\lassoset$, the algorithm starts rejecting significantly faster once it reaches the $5\approxerror$ threshold. 
 
For \acsincome{}, since the FDD distance  
between the two distributions is $0.02$, \SurVerif{} accepts ({\color{Cerulean} \textbf{BLUE}} line) 
when $\approxerror$ increases. In this regime, we observe the predicted $1/\approxerror^2$ decay in the sample complexity ({\color{Green} \textbf{GREEN}} line). 
But when $\approxerror$ goes smaller, \SurVerif{} tends to reject. Specially, when $\approxerror \leq \distname\fbrac{\sampD,\D}/5 = 0.004$, the FDD distance being too far w.r.t. $\approxerror$, the early stopping kicks in and the sample complexity hits a plateau.


In conclusion, \textit{we observe that the effective sample complexity of test decreases as the distance of $\sampD$ from $\D$ increases, and the effective sample complexity of \SurVerif{} is much lower than that of the worst case complexity.} Extended experimental results are presented in the Appendix. 

\section{Discussions, Limitations, and Future Works}\label{section:future work} 
We consider the problem of testing the credibility of survey data when used to develop a regression model. We propose an algorithm, \SurVerif{}, that certifies the data quality by evaluating and testing the \distname{} metric between survey and the true distribution without explicitly reconstructing the models—an approach that, to the best of our knowledge, is novel in the testing literature. 
Notably, \#samples-to-test required by \SurVerif{} is independent of the data dimension, thereby overcoming the curse of dimensionality in this context. 

In this paper, though we provide a general framework for testing credibility, our theoretical analysis focuses exclusively on linear and kernel regression models with bounded response, and homoskedastic, and non-correlated noise, which may limit its applicability. In future, it would be interesting to extend the model-specific credibility testing to regressions with heteroskedastic and correlated noise. Furthermore, it would be interesting to extending the testing framework of our algorithm beyond the regression models with bounded response, i.e. where closed-form Rademacher complexity based generalization bounds are not known. Furthermore, as indicated by the experiments, the proposed framework works for unbounded data coming from tail-bounded distributions. Thus, it will be interesting to extend the theoretical analysis to such settings.

\section*{Acknowledgement}
This work has been supported by the Inria-ISI, Kolkata associate team “SeRAI”. We also
acknowledge the French National Research Agency (ANR) in the framework of the PEPR AI project FOUNDRY (ANR-23-PEIA-0003), and the ANR JCJC for the
REPUBLIC project (ANR-22-CE23-0003-01) for partially supporting this work.


\bibliographystyle{alpha}
\bibliography{refs_NeurIPS.bib}

\newpage
\appendix
\part{Appendix}
\parttoc


\newpage
\section{\distname-variance Decomposition of Loss: Proof of Lemma~\ref{lem:dist}}

\distanceLemma*

\begin{proof}
Observe that 
\begin{align*}
    &~~\E_{(\covariate, \response) \sim \valD}\tbrac{\fbrac{f_\surveyset(\covariate)-\response}^2}\\
    &= \E_{\D, \regnoise}\tbrac{\fbrac{f_\surveyset(\covariate)-\optimalf(\covariate) - \regnoise}^2}\\
    &= \E_{\D, \regnoise}\tbrac{\fbrac{f_\surveyset(\covariate)-\optimalf(\covariate)}^2 + \regnoise^2 -2\fbrac{f_\surveyset(\covariate)- \optimalf(\covariate)}\regnoise}\\
    &= \E_{\D}\tbrac{\fbrac{f_\surveyset(\covariate)- \optimalf(\covariate)}^2} + \E_{\regnoise}\tbrac{\regnoise^2} - 2\E_{\D}\E_{\regnoise}\tbrac{\fbrac{f_\surveyset(\covariate)- \optimalf(\covariate)}\regnoise}\\
    &= \E_{\D}\tbrac{\fbrac{f_\surveyset(\covariate)- \optimalf(\covariate)}^2} + \regnoisevariance + 2\E_{\D}\tbrac{(f_\surveyset(\covariate)- \optimalf(\covariate))\E_{\regnoise}[\regnoise]} &&\text{From Assumption \ref{Assumption : Assumptions on the linear regression model}}\\
    &= \E_{\D}\tbrac{\fbrac{f_\surveyset(\covariate)- \optimalf(\covariate)}^2} + \regnoisevariance\\
    &= \distF^2(f_{\surveyset}, \optimalf) + \regnoisevariance
\end{align*}

\end{proof}
\section{Generic Two-sided Generalization Bounds: Proof of Thoerem~\ref{theorem: Regression Generalization Bound}}

Before proving the theorem, we state the following results that are relevant to our proof:

\begin{lemma}[\textbf{Talagrand's Contraction Lemma}~\cite{LedouxTalagrand/Book/1991/ProbabilityInBanachSpaces}]\label{Lemma: Talagrand's Contraction Lemma}
    Given a real-valued $\losslipschitz$-lipschitz loss function $\loss$, a sample set $\surveyset$ and a hypothesis class $\hypothesisset$ of real valued function, the following inequality holds:
    \begin{align*}
        \empradcomp{\surveyset}{\loss\circ\hypothesisset} \leq \losslipschitz\empradcomp{\surveyset}{\hypothesisset}
    \end{align*}
\end{lemma}

\begin{lemma}[\textbf{McDiarmid's Inequality}~\citep{mohri2018}]\label{Lemma: McDiarmid's Inequality}
    Let $X_1,X_2,\ldots,X_m \in \sX^m$ be iid random variables and there exists a constant $c$ such that $\func{f}{\sX^m}{\R}$ satisfies:
    \begin{align*}
        \abs{f\fbrac{x_1,\ldots,x_i,\ldots,x_m} - f\fbrac{x_1,\ldots,x'_i,\ldots,x_m}} \leq c,\;\;\forall x_1, \ldots, x_m,x'_i \in \covariatedomain 
    \end{align*}
    Then, for any $\approxerror > 0$ the following inequality hold:
    \begin{align*}
        \Pr\tbrac{\abs{f\fbrac{X_1,\ldots,X_m} - \E\tbrac{f\fbrac{X_1,\ldots,X_m}}} \geq \approxerror} \leq 2\exp\fbrac{\frac{-2\approxerror^2}{mc^2}}
    \end{align*}
\end{lemma}

We also introduce the definition of Rademacher Complexity that only depends on the class of functions under consideration

\begin{definition}[\textbf{Rademacher Complexity}~\citep{mohri2018}]\label{Definition: Rademacher Complexity}
    Let $\surveyset$ be a sample set of size $\sizeS \geq 1$ drawn as i.i.d. samples from the distribution $\distribution$. Then, the Rademacher Complexity of $\losshypothesisset$ is the expectation of the empirical Rademacher complexity over all samples of size $\sizeS$ drawn from $\distribution$:
    \begin{align*}
        \empradcomp{\sizeS}{\losshypothesisset} = \E_{\surveyset \sim \distribution^\sizeS}\tbrac{\empradcomp{\surveyset}{\losshypothesisset}}
    \end{align*}
\end{definition}

The next result is the intermediate lemma required, which quantifies how well the empirical mean estimates the true expectation over a bounded function class, in terms of its empirical Rademacher complexity (Definition~\ref{Definition: Empirical Rademacher Complexity}).

\begin{lemma}[\textbf{Two-sided Rademacher Bound for Bounded Functions}]\label{Lemma: Two side rademacher bound}
     Given a family of functions $\losshypothesisset$ containing functions  $\func{\losshypothesis}{\datadomain}{\tbrac{0,\lossbound}}$. Let $\surveyset$ be a sample set of size $\sizeS \geq 1$ drawn as i.i.d. samples from the distribution $\distribution$. Then with probability at least $1 - \confidence$ for all $\losshypothesis \in \losshypothesisset$:
    \begin{align*}
        \abs{\frac{1}{\n}\sum_{i \in [\n]} \losshypothesis(\datapoint) - \E\tbrac{\losshypothesis(\datapoint)}} \leq 2\empradcomp{\surveyset}{\losshypothesisset}+3\lossbound\sqrt{\frac{\log\frac{4}{\confidence}}{2\n}}
    \end{align*}
\end{lemma}

\begin{proof}
    For a given sample set $\surveyset$ of size $\n$, let us denote by $\hat{\E}_\surveyset\tbrac{g}$ the empirical loss $\frac{1}{\n}\sum_{i \in [\n]} \losshypothesis(\datapoint_i)$. Consequently, we define a function $\Phi$ corresponding of a sample set $\surveyset$ as:
    \begin{align*}
        \Phi(\surveyset) = \sup_{\losshypothesis \in \losshypothesisset}  \hat{\E}_\surveyset\tbrac{\losshypothesis} - \E\tbrac{\losshypothesis} 
    \end{align*}
    We first upper bound the expectation of this function $\Phi(\surveyset)$ over $\surveyset \in \distribution^\n$.
    \begin{align}
        \nonumber\E_\surveyset\tbrac{\Phi(\surveyset)}
        =&\E_\surveyset\tbrac{\sup_{\losshypothesis \in \losshypothesisset} \hat{\E}_\surveyset\tbrac{\losshypothesis} - \E\tbrac{\losshypothesis} }\\\nonumber
        =&\E_\surveyset\tbrac{\sup_{\losshypothesis \in \losshypothesisset} \E_{\surveyset'}\tbrac{\hat{\E}_{\surveyset}\tbrac{\losshypothesis} - \hat{\E}_{\surveyset'}\tbrac{\losshypothesis}}}\\\nonumber
        \leq&\E_{\surveyset,\surveyset'}\tbrac{\sup_{\losshypothesis \in \losshypothesisset} \hat{\E}_{\surveyset}\tbrac{\losshypothesis} - \hat{\E}_{\surveyset'}\tbrac{\losshypothesis}} &\text{By Jensen's Inequality}\\\nonumber
        =&\E_{\surveyset,\surveyset'} \tbrac{\sup_{\losshypothesis \in \losshypothesisset} \frac{1}{\n}\sum_{i \in [\n]} \fbrac{\losshypothesis(\datapoint_i) - \losshypothesis(\datapoint'_i)}}\\\nonumber
        =&\E_{\rademacher,\surveyset,\surveyset'} \tbrac{\sup_{\losshypothesis \in \losshypothesisset} \frac{1}{\n}\sum_{i \in [\n]} \rademacher_i\fbrac{\losshypothesis(\datapoint_i) - \losshypothesis(\datapoint'_i)}}&\text{Expectation over $\surveyset,\surveyset'$}\\\nonumber
        \leq&\E_{\rademacher,\surveyset,\surveyset'} \tbrac{\sup_{\losshypothesis \in \losshypothesisset} \frac{1}{\n}\sum_{i \in [\n]} \rademacher_i\losshypothesis(\datapoint_i) + \sup_{\losshypothesis \in \losshypothesisset}  \frac{1}{\n}\sum_{i \in [\n]} -\rademacher_i\losshypothesis(\datapoint'_i)}&\sup(a+b) \leq \sup(a)+\sup(b)\\\nonumber
        =&\E_{\rademacher,\surveyset'} \tbrac{\sup_{\losshypothesis \in \losshypothesisset} \frac{1}{\n}\sum_{i \in [\n]} \rademacher_i\losshypothesis(\datapoint_i)} + \E_{\rademacher,\surveyset} \tbrac{ \sup_{\losshypothesis \in \losshypothesisset}  \frac{1}{\n}\sum_{i \in [\n]} -\rademacher_i\losshypothesis(\datapoint'_i)}&\text{Linearity of expectations}\\
        =&2\E_{\rademacher,\surveyset} \tbrac{\sup_{\losshypothesis \in \losshypothesisset} \frac{1}{\n}\sum_{i \in [\n]} \rademacher_i\losshypothesis(\datapoint_i)} = 2\radcomp{\losshypothesisset}~\label{Equation: Expectation is RadComp}
    \end{align}
    Now, we will use the McDiarmid's inequality(Lemma~\ref{Lemma: McDiarmid's Inequality}) on this function. For that purpose, observe that each coordinate of the input essentially corresponds to one of the data points in the sample. We use this fact and the boundedness of $\losshypothesis$ to obtain our prerequisite bound to apply McDiarmid's inequality. Let us consider two sample sets $\surveyset$ and $\surveyset'$ that differs at exactly one sample point, say the $i$-th location. Then, we have:
    \begin{align*}
        \Phi\fbrac{\surveyset} - \Phi\fbrac{\surveyset'}
        \leq \sup_{\losshypothesis \in \losshypothesisset} \fbrac{\hat{\E}_\surveyset\tbrac{\losshypothesis} - \hat{\E}_{\surveyset'}\tbrac{\losshypothesis}}
        = \sup_{\losshypothesis \in \losshypothesisset} \frac{\losshypothesis(\datapoint_i) - \losshypothesis(\datapoint'_i)}{\n}
        \leq\frac{\lossbound}{\n}
    \end{align*}
    
    Here, the first inequality follows from the fact that $\sup_x\fbrac{f(x)-g(x)} \geq\sup_x f(x) - \sup_x g(x)$, and
    Now, by McDiarmid's inequality, we have,
    \begin{align}
        \Pr\tbrac{\Phi(\surveyset)-\E\tbrac{\Phi(\surveyset)} \geq \lossbound\sqrt{\frac{\log\frac{4}{\confidence}}{2\n}}} \leq \exp\fbrac{\frac{-2\lossbound^2\n^2\log\frac{4}{\confidence}}{2\n^2\lossbound^2}} = \frac{\confidence}{4}\label{Eq: Loss Deviation Bound}
    \end{align}
    Combining equations~\eqref{Equation: Expectation is RadComp} and~\eqref{Eq: Loss Deviation Bound}, we have for all $\losshypothesis \in \losshypothesisset$:
    \begin{align}
        \Pr\tbrac{\frac{1}{\n}\sum_{i \in [\n]} \losshypothesis(\datapoint) - \E\tbrac{\losshypothesis(\datapoint)} \geq 2\radcomp{\losshypothesisset}+\lossbound\sqrt{\frac{\log\frac{4}{\confidence}}{2\n}}} \leq \frac{\confidence}{4}\label{Equation: Rademacher Generalization Bound}
    \end{align}
    Now, we bound the empirical Rademacher sample complexity in terms of the Rademacher complexity. We again consider two sample sets $\surveyset$, and $\surveyset'$ that differs at exactly one point, say $\datapoint_i$. Then, using the fact that $\sup_x\fbrac{f(x)-g(x)} \geq\sup_x f(x) - \sup_x g(x)$, we get 
    \begin{align*}
        \empradcomp{\surveyset}{\losshypothesisset} - \empradcomp{\surveyset'}{\losshypothesisset}
        \leq&\frac{1}{\n}\E\tbrac{\sup_{\losshypothesis \in \losshypothesisset} \fbrac{\rademacher_i\fbrac{\losshypothesis(\datapoint_i)-\losshypothesis(\datapoint'_i)}}}
        \leq\frac{\lossbound}{\n}
    \end{align*}
    Now, by McDiarmid's Ineqality(Lemma~\ref{Lemma: McDiarmid's Inequality}), we have:
    \begin{align}
        \Pr\tbrac{\radcomp{\losshypothesisset}-\empradcomp{\surveyset}{\losshypothesisset} \geq \lossbound\sqrt{\frac{\log\frac{4}{\confidence}}{2\n}}} \leq \frac{\confidence}{4}\label{Eq: Empirical and True Rademacher Complexity}
    \end{align}
    Now, we combine equations~\eqref{Equation: Rademacher Generalization Bound} and~\eqref{Eq: Empirical and True Rademacher Complexity} through an union bound to obtain:
    \begin{align*}
        \Pr\tbrac{\frac{1}{\n}\sum_{i \in [\n]} \losshypothesis(\datapoint) - \E\tbrac{\losshypothesis(\datapoint)} \geq 2\empradcomp{\surveyset}{\losshypothesisset}+3\lossbound\sqrt{\frac{\log\frac{4}{\confidence}}{2\n}}} \leq \frac{\confidence}{2}
    \end{align*}
    Similarly, we can show:
    \begin{align*}
        \Pr\tbrac{\E\tbrac{\losshypothesis(\datapoint)} - \frac{1}{\n}\sum_{i \in [\n]} \losshypothesis(\datapoint) \geq 2\empradcomp{\surveyset}{\losshypothesisset}+3\lossbound\sqrt{\frac{\log\frac{4}{\confidence}}{2\n}}} \leq \frac{\confidence}{2}
    \end{align*}
    Combining through a union bound, we get the desired result.
\end{proof}

\noindent\textbf{Proof of Theorem~\ref{theorem: Regression Generalization Bound}.} Now, we are ready to give the proof of Theorem~\ref{theorem: Regression Generalization Bound}. A restatement of the theorem is given below.

\twoSidedGeneralizationBound*
\begin{proof}From Lemma~\ref{Lemma: Two side rademacher bound}, we know the two-sided deviation on the empirical mean w.r.t true expection over a bounded function class $\losshypothesisset$ containing functions  $\func{\losshypothesis}{\datadomain}{\tbrac{0,\lossbound}}$,
\begin{equation}
    \Pr\tbrac{\abs{\frac{1}{\n}\sum_{i \in [\n]} \losshypothesis(\datapoint) - \E\tbrac{\losshypothesis(\datapoint)}} \leq 2\empradcomp{\surveyset}{\losshypothesisset}+3\lossbound\sqrt{\frac{\log\frac{4}{\confidence}}{2\n}}} \geq 1 - \confidence
    \label{ineq: general error}
\end{equation}
Take $\losshypothesisset$ to be the set of loss functions $\func{\loss}{\yrange\times\yrange}{\tbrac{0,\lossbound}}$, then for any $f \in \hypothesisset$ we can write inequality~\eqref{ineq: general loss} as,
\begin{equation}
    \Pr\tbrac{\abs{\E_{\fbrac{\covariate,\response}\sim\distribution}\tbrac{\loss\fbrac{\hypothesis(\covariate),\response}} - \frac{1}{\n}\sum_{(\covariate, \response) \in \surveyset} \loss\fbrac{\hypothesis(\covariate),\response}} \leq 2\empradcomp{\surveyset}{\loss\circ\hypothesisset} + 3\lossbound\sqrt{\frac{\log\frac{4}{\confidence}}{2\n}}  } \geq 1 - \confidence
    \label{ineq: general loss}
\end{equation}
    From Talagrand's Contraction Lemma~\ref{Lemma: Talagrand's Contraction Lemma}, we have
    \begin{equation}\label{ineq: Talagrand}
        \empradcomp{\surveyset}{\loss\circ\hypothesisset} \leq \losslipschitz\empradcomp{\surveyset}{\hypothesisset}
    \end{equation}
     
    Plugging back inequality \eqref{ineq: Talagrand} in~\eqref{ineq: general loss} we get the following with probability at least $1 - \confidence$:
    \begin{equation*}
    \Pr\tbrac{\abs{\E_{\fbrac{\covariate,\response}\sim\distribution}\tbrac{\loss\fbrac{\hypothesis(\covariate),\response}} - \frac{1}{\n}\sum_{(\covariate, \response) \in \surveyset} \loss\fbrac{\hypothesis(\covariate),\response}} \leq 2\losslipschitz\empradcomp{\surveyset}{\hypothesisset} + 3\lossbound\sqrt{\frac{\log\frac{4}{\confidence}}{2\n}}} \geq 1 - \confidence
    \end{equation*}
    this completes the proof.

\end{proof}


\section{Minimum Survey Size for Learning Noise Variance: Proof of Lemma~\ref{lemma: sample complexity noise variance}}
This section is organized into three parts. In subsection~\ref{subsection: Generalization Bound for Noise Variance}, we establish a two-sided generalization bound for the estimation of noise variance in the regression model, in terms of empirical rademacher complexity.  In subsection~\ref{subsection: Generalization Error Bound for Bounded Hypothesis Class} we extends this result to both bounded linear and kernel hypothesis classes using their corresponding rademacher bounds. Finally, In subsection~\ref{subsection: final proof} presents the proof of Lemma~\ref{lemma: sample complexity noise variance}, which formalizes the minimum survey size required for estimating noise variance.

\subsection{From Two-sided Generalization Bound to Estimating Noise Variance}\label{subsection: Generalization Bound for Noise Variance}
The next result uses Theorem~\ref{theorem: Regression Generalization Bound} applied to the squared loss setting, and combines it with the assumptions specific to linear regression over bounded domains to estimate the noise variance. 

\begin{lemma}[\textbf{Concentration of Empirical Squared Loss around Noise Variance}]\label{Lemma: Two-sided Generalization Bound for Noise Variance}
   Given a linear regression model $\optimalf: \response = \optimalf(\covariate) + \regnoise$, where $\regnoise$ is the zero-mean additive noise term with variance $\Var_{(\covariate,\response) \sim \valD} (\regnoise) = \regnoisevariance$. Let $\surveyset$ be a sample set of size $\sizeS \geq 1$ drawn as i.i.d. samples from the distribution $\sampD$. If Assumptions \ref{Assumption : Assumptions on the linear regression model} and \ref{Assumption: Bounded Variables} holds, then the regression model $f_{\surveyset} \gets \argmin_{f \in \boundedlinfunc} \frac{1}{\sizeS} \sum_{(\covariate,\response) \in \surveyset} (f(\covariate) - \response)^2$ satisfies, with probability at least $1 - \confidence$: 
    \begin{align*}
        \abs{\regnoisevariance - \frac{1}{\sizeS}\sum_{(\covariate, \response) \in \surveyset} (f_\surveyset(\covariate)-\response)^2} \leq  8\widehat{\sR}_S(\boundedlinfunc) + 12\sqrt{\frac{\log\frac{4}{\confidence}}{2\sizeS}}
    \end{align*}
\end{lemma}

\begin{proof} We split the proof into two steps,

\noindent\textbf{Step 1: Generalization bound for squared loss.} In this step, we show the two-sided generalization bound for squared loss. From Theorem~\ref{theorem: Regression Generalization Bound}, We consider the squared loss $\loss(f(\covariate), \response) = (f(\covariate) - y)^2$ for all $(\covariate, \response) \in \xdomain \times \yrange $ and $f \in \boundedlinfunc$. 

From assumption~\ref{Assumption: Bounded Variables}, we bound the maximum value of the loss function:
\begin{align}\label{ineq: max loss 1}
        \lossbound = \max_{\response,\response'}\loss(\response,\response') = \max_{\response,\response'}\fbrac{\response - \response'}^2 \leq 4
    \end{align}
Similarly, for the Lipschitzness of the loss function, we have:
    \begin{align}\label{ineq: max loss 2}
        \losslipschitz \leq \max_{\response,\response'} \nabla_\response \loss\fbrac{\response,\response'} = \max_{\response,\response'} \nabla_\response \fbrac{\response-\response'}^2 = \max_{\response,\response'} 2\fbrac{\response-\response'} \leq 4
    \end{align}
Now, applying Theorem~\ref{theorem: Regression Generalization Bound} to the squared loss and function class $\boundedlinfunc$, we get,
\begin{equation}
    \Pr \tbrac{\abs{\E_{(\covariate,\response) \sim \sD} \left[ ( f(\covariate)-\response)^2 \right] - \frac{1}{\sizeS}\sum_{(\covariate, \response) \in \surveyset} (f(\covariate)-\response)^2}
        \leq  8\widehat{\sR}_S(\boundedlinfunc) + 12\sqrt{\frac{\log\frac{4}{\confidence}}{2\sizeS}}} \geq 1 - \confidence
    \label{ineq: squared loss concentration}
\end{equation}

\noindent\textbf{Step 2: Concentration of noise variance.} We now show that the empirical squared loss of the estimator $f_\surveyset$ concentrates around the true noise variance $\regnoisevariance$, using the generalization bound from Step 1 and Assumption~\ref{Assumption : Assumptions on the linear regression model}. We prove the upper and lower bounds separately.

\noindent\textbf{Upper Bound:} Let $\optimalf$ is the optimal linear regression model on the true distribution $\D$. Let $\optimalf_\surveyset$ is the optimal linear regression model on the survey distribution $\sampD$. 

    From Assumption \ref{Assumption : Assumptions on the linear regression model}, we have: 
     \begin{align}\label{Eq: Exp of optimal validation function is noise}
        \regnoisevariance = \E_{(\covariate,\response) \sim \D} [(\optimalf(\covariate) - \response)^2] =  \E_{(\covariate,\response) \sim \sampD} [(\optimalf_\surveyset(\covariate) - \response)^2]
    \end{align}

Therefore, 
        \begin{align}
            \regnoisevariance = \E_{(\covariate,\response) \sim \sampD}[(\optimalsurveyf(\covariate)-\response)^2] \leq \E_{(\covariate,\response) \sim \sampD}[(\surveyf(\covariate)-\response)^2]          
            \label{Eq: Exp of optimal survey function is noise}
        \end{align}

The inequality~\eqref{Eq: Exp of optimal survey function is noise} comes from the optimality of $\optimalf_\surveyset$ on $\sampD$. 

Now, applying the upper-sided generalization bound in~\eqref{ineq: squared loss concentration} with $f = f_\surveyset$, we have
\begin{equation}
    \Pr\tbrac{\E_{(\covariate,\response) \sim \sampD}[(\surveyf(\covariate)-\response)^2]  \leq \frac{1}{\sizeS}\sum_{(\covariate, \response) \in \surveyset} (f_\surveyset(\covariate)-\response)^2 + 8\widehat{\sR}_S(\boundedlinfunc) + 12\sqrt{\frac{\log\frac{4}{\confidence}}{2\sizeS}}} \geq 1 - \frac{\confidence}{2}
    \label{ineq: upper side bound w.r.t f_S}
\end{equation}

Combining \eqref{Eq: Exp of optimal survey function is noise} and \eqref{ineq: upper side bound w.r.t f_S} we get,
\begin{equation}\label{ineq: upper bound on noise}
    \Pr \tbrac{\regnoisevariance  \leq \frac{1}{\sizeS}\sum_{(\covariate, \response) \in \surveyset} (f_\surveyset(\covariate)-\response)^2 + 8\widehat{\sR}_S(\boundedlinfunc) + 12\sqrt{\frac{\log\frac{4}{\confidence}}{2\sizeS}}} \geq 1 - \frac{\confidence}{2}
\end{equation}

\noindent\textbf{Lower Bound:} We now show the lower bound, by applying lower-sided generalization bound in~\eqref{ineq: squared loss concentration} with $f = \optimalf_\surveyset$ we get,
\begin{equation}\label{ineq: temp lower bound on noise variance}
   \Pr \tbrac{ \E_{(\covariate,\response) \sim \sampD} \left[ ( {\optimalf_\surveyset(\covariate)-\response})^2 \right] \geq \frac{1}{\sizeS}\sum_{(\covariate, \response) \in \surveyset} (\optimalf_\surveyset(\covariate)-\response)^2 - 8\widehat{\sR}_S(\boundedlinfunc) - 12\sqrt{\frac{\log\frac{4}{\confidence}}{2\sizeS}} } \geq \ 1- \frac{\confidence}{2}
\end{equation}
Since $f_\surveyset$ is the optimal regression model over the empirical loss. Therefore,
\begin{equation}\label{ineq: optimalility of f_S}
   \sum_{(\covariate, \response) \in \surveyset} (\optimalf_\surveyset(\covariate)-\response)^2 \geq \sum_{(\covariate, \response) \in \surveyset} (f_\surveyset(\covariate)-\response)^2
\end{equation}

Using inequality \eqref{ineq: optimalility of f_S} in \eqref{ineq: temp lower bound on noise variance}, we get 
\begin{equation}
    \label{ineq: xyz}
    \Pr \tbrac{\E_{(\covariate,\response) \sim \sD} \left[ ( {\optimalf_\surveyset(\covariate)-\response})^2 \right] \geq \frac{1}{\sizeS}\sum_{i \in [\sizeS]} (f_\surveyset(\covariate)-\response)^2 - 8\widehat{\sR}_S(\boundedlinfunc) - 12\sqrt{\frac{\log\frac{4}{\confidence}}{2\sizeS}}} \geq 1 - \frac{\confidence}{2}
\end{equation}
Now, Combining \eqref{Eq: Exp of optimal validation function is noise} and \eqref{ineq: xyz} we get,
\begin{equation}\label{ineq: lower bound on noise}
    \Pr \tbrac{\regnoisevariance \geq \frac{1}{\sizeS}\sum_{(\covariate, \response) \in \surveyset} (f_\surveyset(\covariate)-\response)^2 - 8\widehat{\sR}_S(\boundedlinfunc) - 12\sqrt{\frac{\log\frac{4}{\confidence}}{2\sizeS}}} \geq 1 - \frac{\confidence}{2}
\end{equation}

Combining the upper bound~\eqref{ineq: upper bound on noise} and lower bound~\eqref{ineq: lower bound on noise} using the union bound we get, with probability at least $1 - \confidence$:
\[
\left| \regnoisevariance - \frac{1}{\sizeS} \sum_{(\covariate,\response) \in \surveyset} (f_\surveyset(\covariate) - \response)^2 \right| \leq 8\widehat{\sR}_S(\boundedlinfunc) + 12 \sqrt{\frac{\log\frac{4}{\confidence}}{2\sizeS}}.
\]
This completes the proof.
\end{proof}

\subsection{From Generalization Bound to Noise variance for Linear and Kernel Classes}\label{subsection: Generalization Error Bound for Bounded Hypothesis Class}
In this section, We show the general two-sided generalization bound for the empirical squared loss from Lemma~\ref{Lemma: Two-sided Generalization Bound for Noise Variance} for specific families of hypothesis classes. In particular, we consider:

\begin{itemize}
    \item Linear function classes with bounded $\ell_1$ and $\ell_2$ norms, corresponding to $\lasso$ and $\ridge$ regression respectively.
    \item Kernel-based functions classes with bounded RKHS norm, corresponding to $\kerna$.
\end{itemize}

In each case, we use upper bounds on the Rademacher complexity for the corresponding class, and then apply Lemma~\ref{Lemma: Two-sided Generalization Bound for Noise Variance} to obtain corresponding generalization guarantees.

\noindent\textbf{Case: $\boundedlinfunc_1(\lasso)$ and $\boundedlinfunc_2(\ridge)$}

\cite{awasthi2020rademacher} has proved the following upper bound of the empirical Rademacher complexity for bounded linear hypothesis classes.

\begin{lemma}[\textbf{Empirical Rademacher Complexity of Bounded Linear Hypothesis }~\cite{awasthi2020rademacher}]\label{lemma:Rademacher Complexity of linear hypothesis}
Let $\boundedlinfunc_p = \{\covariate \rightarrow \inn{\coeff}{\covariate}:\norm{\coeff}_p \leq \linweight\}$ be a family of linear functions defined over $\R^d$ with bounded weight in $\ell_p$-norm where $p \in \{1,2\}$. Let $\surveyset = \fbrac{\covariate_1,\covariate_2,\ldots,\covariate_{\sizeS}}$ be a sample of size $\sizeS$. Then, the empirical Rademacher complexity of $\boundedlinfunc_p$ is upper bounded by:
    \begin{equation*}
        \widehat{\sR}_S(\boundedlinfunc_p) \leq \begin{cases}
            \frac{\linweight}{\sizeS} \sqrt{2 \log(2d)}\norm{\mathbf{X}^T}_{2,\infty} &\text{if $p = 1$}\\
            \frac{\linweight}{\sizeS} \norm{\mathbf{X}^T}_{2,2} &\text{if $p = 2$}\\
        \end{cases}
    \end{equation*}
where $\mathbf{X}$ is a $d \times \sizeS$ matrix with $\covariatei_i$'s as columns: $\mathbf{X} = \tbrac{\covariate_1 \ldots \covariate_{\sizeS}}$
\end{lemma}

\begin{lemma}[\textbf{Two-sided Generalization Bound of $\ell_1$ and $\ell_2$ bounded linear hypothesis class}]\label{lemma: two-sided generalization error l_1 and l_2}
    Let $\boundedlinfunc_p = \{\covariate \rightarrow \inn{\coeff}{\covariate}:\norm{\coeff}_p \leq \linweight\}$. Given a linear regression model $\optimalf: \response = \optimalf(\covariate) + \regnoise$, where $\regnoise$ is the zero-mean additive noise term with variance $\Var_{(\covariate,\response) \sim \D} (\regnoise) = \regnoisevariance$.
    Given a sample $\surveyset$ of size $\sizeS$ sampled i.i.d from a distribution $\sampD$. If Assumption \ref{Assumption : Assumptions on the linear regression model} and \ref{Assumption: Bounded Variables} holds, then the regression model
    \[
    f_{\surveyset} \gets \argmin_{f \in \boundedlinfunc_p} \frac{1}{\sizeS} \sum_{(\covariate,\response) \in \surveyset} (f(\covariate) - \response)^2
    \]
    satisfies, with probability at least $1 - \confidence$:
    \begin{equation*}
        \abs{\regnoisevariance - \frac{1}{\sizeS}\sum_{(\covariate, \response) \in \surveyset} (f_\surveyset(\covariate)-\response)^2}\leq 
    \begin{cases}
         8 \sqrt{\frac{2 \log(2 d)}{\sizeS}}+ 12\sqrt{\frac{\log\frac{4}{\confidence}}{2\sizeS}} &\text{If $p =1 \; (\lasso)$}\\

        8 \sqrt{\frac{d}{\sizeS}} + 12\sqrt{\frac{\log\frac{4}{\confidence}}{2\sizeS}} &\text{If $p = 2\;(\ridge)$}\\
    \end{cases}
    \end{equation*}
\end{lemma}
\begin{proof} From Assumption~\ref{Assumption: Bounded Variables}, we have $\norm{\covariate}_\infty \leq 1$ for all $\covariate \in \xdomain$. Therefore, each column $\covariate_i$ of the matrix $\mathbf{X} \in \R^{d \times \sizeS}$ satisfies $\norm{\covariate_i}_\infty \leq 1$. then $\norm{\mathbf{X}^T}_{2,\infty} \leq \sqrt{\sizeS}$ and $\norm{\mathbf{X}^T}_{2,2} \leq \sqrt{d\sizeS}$.

    \textbf{For $p = 1$ ($\lasso$):} From Lemma~\ref{lemma:Rademacher Complexity of linear hypothesis}, we get:
    \begin{align*}
        \widehat{\sR}_S(\boundedlinfunc_1) 
        &\leq \frac{1}{\sizeS} \sqrt{2 \log(2d)} \norm{\mathbf{X}^T}_{2,\infty} \\
        &\leq \frac{1}{\sizeS} \sqrt{2 \log(2d)} \cdot \sqrt{\sizeS} =  \sqrt{\frac{2 \log(2d)}{\sizeS}}
    \end{align*}

    \textbf{For $p = 2$ ($\ridge$):} Again, from Lemma~\ref{lemma:Rademacher Complexity of linear hypothesis},
    \begin{align*}
        \widehat{\sR}_S(\boundedlinfunc_2)
        &\leq \frac{1}{\sizeS} \norm{\mathbf{X}^T}_{2,2} \leq \frac{\linweight}{\sizeS} \sqrt{d \sizeS} =  \sqrt{\frac{d}{\sizeS}}
    \end{align*}

    Now, plugging the above bounds on $\widehat{\sR}_S(\boundedlinfunc_p)$ into Lemma~\ref{Lemma: Two-sided Generalization Bound for Noise Variance} yields:

    \[
    \abs{\regnoisevariance - \frac{1}{\sizeS} \sum_{(\covariate, \response) \in \surveyset} (f_\surveyset(\covariate) - \response)^2}
    \leq 8 \widehat{\sR}_S(\boundedlinfunc_p) + 12\sqrt{\frac{\log\frac{4}{\confidence}}{2\sizeS}}
    \]

    which gives the desired bounds:
    \begin{align*}
        \text{$\lasso$:}\quad &8  \sqrt{\frac{2 \log(2d)}{\sizeS}} + 12\sqrt{\frac{\log\frac{4}{\confidence}}{2\sizeS}} \\
        \text{$\ridge$:}\quad &8  \sqrt{\frac{d}{\sizeS}} + 12\sqrt{\frac{\log\frac{4}{\confidence}}{2\sizeS}}
    \end{align*}

\end{proof}

\noindent\textbf{Case: $\kernelset(\kerna)$}

We define the hypothesis class as:
\[
\kernelset \defeq \left\{ \covariate \mapsto \inn{\coeff}{\kernelfunction(\covariate)}_{\hilbertspace} : \norm{\coeff}_{\hilbertspace} \leq 1 \right\}
\]
where $\kernelfunction : \xdomain \to \hilbertspace$ is the feature map associated with a positive definite symmetric (PDS) kernel $\kernel : \xdomain \times \xdomain \to \R$.

We first recall the following Rademacher complexity bound for kernel regression from~\cite[Theorem 6.12]{mohri2018}:

\begin{lemma}[\textbf{PDS Kernel Rademacher Complexity Bound}~\cite{mohri2018}]\label{Lemma: PDS Kernel Rademacher Complexity Bound}
    Let $\kernel$ be a PDS kernel with associated feature map $\kernelfunction$ satisfying $\kernel(\covariate, \covariate) = \norm{\kernelfunction(\covariate)}^2_{\hilbertspace} \leq \kernelnormbound^2$ for all $\covariate \in \xdomain$. Then, for any i.i.d. sample $\surveyset$ of size $\n$, the empirical Rademacher complexity of $\kernelset$ satisfies:
    \[
    \empradcomp{\surveyset}{\kernelset} \leq \frac{\kernelnormbound}{\sqrt{\n}}
    \]
\end{lemma}

\begin{lemma}[\textbf{Two-sided Generalization Error of Bounded Kernel Hypothesis}]\label{lemma: Kernel Noise Variance}
    Let $\kernelset$ be defined as above and suppose $\kernel(\covariate, \covariate) \leq \kernelnormbound^2$ for all $\covariate \in \xdomain$. Given a linear regression model $\optimalf: \response = \optimalf(\covariate) + \regnoise$, where $\regnoise$ is the zero-mean additive noise term with variance $\Var_{(\covariate,\response) \sim \D} (\regnoise) = \regnoisevariance$ and sample $\surveyset$ of size $\sizeS$ sampled i.i.d from a distribution $\sampD$. If Assumptions \ref{Assumption : Assumptions on the linear regression model} and \ref{Assumption: Bounded Variables} holds, then the regression model 
    \[f_{\surveyset} \gets \argmin_{f \in \kernelset} \frac{1}{\sizeS} \sum_{(\covariate,\response) \in \surveyset} (f(\covariate) - \response)^2\]
    satisfies, with probability at least $1 - \confidence$:
    \[
    \abs{\regnoisevariance - \frac{1}{\n}\sum_{(\covariate, \response) \in \surveyset} (f_\surveyset(\covariate)-\response)^2}
    \leq 8  \frac{\kernelnormbound}{\sqrt{\n}} + 12 \sqrt{\frac{\log\frac{4}{\confidence}}{2\n}}.
    \]
\end{lemma}

\begin{proof}
    From Lemma~\ref{Lemma: PDS Kernel Rademacher Complexity Bound}, we have:
    \[
    \empradcomp{\surveyset}{\kernelset} \leq \frac{\kernelnormbound}{\sqrt{\n}}.
    \]
    Plugging this into the general bound from Lemma~\ref{Lemma: Two-sided Generalization Bound for Noise Variance} obtains the stated result.
\end{proof}

\subsection{From Noise Variance Bounds to Minimum Survey Size: Proof of Lemma~\ref{lemma: sample complexity noise variance}}\label{subsection: final proof}

We now translate the generalization error bounds derived for $\lasso, \ridge$ and $\kerna$ into sample size guarantees for estimating the noise variance, which leads to the proof of Lemma~\ref{lemma: sample complexity noise variance}.

\sampleComplexityNoiseVariance*
\begin{table}[ht!]
      \caption{Sufficient Size of Survey Data 
      }
      \centering
      \begin{tabular}{llll}
        \toprule
        Hypothesis Class($\boundedlinfunc$) & $\boundedlinfunc_1$(\lasso)     & $\boundedlinfunc_2$(\ridge)     & $ \kernelset
        $(\kerna) \\
        \midrule
        Size of $\surveyset$ ($\sbound(\regressionset)$) & $\Omega\fbrac{\frac{\log(d)} {\tol^2}}$ & $\Omega\fbrac{\frac{d}{\tol^2}}$  & $\Omega\fbrac{\frac{\kernelnormbound^2}{\tol^2}}$    \\
        \bottomrule
      \end{tabular}
    \end{table}

\begin{proof}
    From Lemmas~\ref{lemma: two-sided generalization error l_1 and l_2} and~\ref{lemma: Kernel Noise Variance}, we have that with probability at least $1 - \confidence$,
    \begin{equation}\label{ineq: all error}
       \abs{\regnoisevariance - \frac{1}{\sizeS}\sum_{(\covariate, \response) \in \surveyset} (f_\surveyset(\covariate)-\response)^2} \leq 
    \begin{cases}
         8  \sqrt{\frac{2 \log(2 d)}{\sizeS}}+ 12\sqrt{\frac{\log\frac{4}{\confidence}}{2\sizeS}} &\text{if $p = 1\;(\lasso)$}\\
        8  \sqrt{\frac{d}{\sizeS}} + 12\sqrt{\frac{\log\frac{4}{\confidence}}{2\sizeS}} &\text{if $p = 2\;(\ridge)$}\\
        8\frac{\kernelnormbound}{\sqrt{\sizeS}} + 12\sqrt{\frac{\log\frac{4}{\confidence}}{2\sizeS}} &\text{$\kerna$}
    \end{cases} 
    \end{equation}

    We now choose $\sizeS$ large enough so that each term on the right-hand side of~\eqref{ineq: all error} is at most $\tol/20$, ensuring the total bound is at most $\tol/10$.

    \noindent\textbf{For $p = 1$ ($\lasso$):} Set
    \[
    \sizeS \geq \max\left\{51200 \frac{\log(2d)}{\tol^2},\; 28800  \frac{\log(4/\confidence)}{\tol^2} \right\}
    \]
    Then,
    \begin{align*}
        \sizeS \geq 51200  \frac{\log(2d)}{\tol^2} &\implies 8 \sqrt{\frac{2 \log(2 d)}{\sizeS}} \leq \frac{\tol}{20}\\
        \sizeS \geq 28800 \frac{\log(4/\confidence)}{\tol^2} &\implies 12\sqrt{\frac{\log\frac{4}{\confidence}}{2\sizeS}} \leq \frac{\tol}{20}
    \end{align*}
    Summing the two terms in~\eqref{ineq: all error} gives a bound of at most $\tol/10$.

    \noindent\textbf{For $p = 2\; (\ridge)$:} Similarly, taking
    \[
    \sizeS \geq \max\left\{51200  \frac{d}{\tol^2},\; 28800 \frac{\log(4/\confidence)}{\tol^2} \right\}
    \]
    yields the desired bound.

    \noindent\textbf{For $\kerna$:} Taking
    \[
    \sizeS \geq \max\left\{25600 \frac{\kernelnormbound^2}{\tol^2},\; 28800 \frac{\log(4/\confidence)}{\tol^2} \right\}
    \]
    ensures that each term on the right-hand side of the kernel bound in~\eqref{ineq: all error} is at most $\tol/20$, completing the proof.
\end{proof}
\section{Correctness of \SurVerif{} and Sample Complexity: Proof of Theorem~\ref{Theorem: Tester Accepts or Rejects}}
Now, we present the proof of the correctness of our algorithm. The restatement of the theorem is given below.
\surverifyCorrectness*

\begin{proof}
The sample complexity of the algorithm can be easily seen from the algorithm. The main thing to prove is the correctness of the algorithm. We will prove the two parts separately. To start, we observe that from Lemma~\ref{lemma: sample complexity noise variance}, we have
\begin{equation}\label{ineq: variance error}
    \Pr \left[ \left| \regnoisevariance - \hat{L}_{S} \right| > 0.1\tol \right] \leq \frac{\confidence}{3}
\end{equation}

Let $\estimatevarD_t$ be the value of $\estimatevarD$ are $t$ rounds of the \textbf{while} loop.  From the Linearity of Expectation, we have
    \begin{equation}\label{eq: unbiased estimator of loss}
    \E\tbrac{\estimatevarD_t} = \sum_{(\covariate,\response) \in \surveyset_{\D}} \E_{(\covariate, \response) \sim \D}\tbrac{(f_\surveyset(\covariate) - \response)^2} = \sampleD \E_{(\covariate, \response) \sim \D}\tbrac{(f_\surveyset(\covariate) - \response)^2}
    \end{equation}

    From Assumption \ref{Assumption: Bounded Variables}, we have $(f_\surveyset(\covariatei_i) - \response)^2 \in [0, 4] $ for all $i \in [t]$. Since each of the $\sampleD$ independent variables is bounded, we now use Hoeffding's inequality to bound the deviation of $\estimatevarD_t$ from its expectation.

\textbf{Proof of \ref{subTheorem: close statement}}  In this case we have to bound the probability that \SurVerif{} outputs \textbf{REJECT} at any of the $t$-iterations of the \textbf{while} loop or in the \textbf{if} statement at the end (line~\ref{SurVerif Line 13} to \ref{SurVerif Line 16}). 

by Hoeffding's Lemma, at any round $\sampleD \in [\totalsampD]$, we have
\begin{equation}\label{ineq: back calculated deviation}
    \Pr\tbrac{\estimatevarD - \sampleD\E_{(\covariate, \response) \sim \D}\tbrac{\fbrac{f_\surveyset(\covariate)- \response}^2 } > \sqrt{\sampleD}\sqrt{2\log\fbrac{\frac{3\totalsampD}{\confidence}}}  } \leq  \exp{\fbrac{-\log\fbrac{\frac{3\totalsampD}{\confidence}}}  } = \frac{\confidence}{3\totalsampD}
\end{equation}

If $ \distF^2(f_\surveyset, \optimalf) \leq \tol$, then from Lemma~\ref{lem:dist}
 \begin{equation}\label{eq:someeq}
\E_{(\covariate, \response) \sim \D}\tbrac{\fbrac{f_\surveyset(\covariate)- \response}^2 } =   \distF(f_\surveyset, \optimalf) + \regnoisevariance \leq \epsilon + \regnoisevariance
\end{equation}

Thus if $ \distF^2(f_\surveyset, \optimalf) \leq \tol$ and $|\regnoisevariance - \hat{L}_S| \leq  0.1\tol$ then 
\begin{equation}\label{ineq:var}
\E_{(\covariate, \response) \sim \D}\tbrac{\fbrac{f_\surveyset(\covariate)- \response}^2 }   \leq 1.1\epsilon + \hat{L}_S
\end{equation}

Combining \eqref{ineq: back calculated deviation} and \eqref{ineq:var} using union bound, we get at any round $\sampleD$,
\begin{equation}\label{eq:while}
    \Pr\tbrac{\estimatevarD - \sampleD\hat{L}_S> 1.1 t\tol + \sqrt{2\sampleD\log\fbrac{\frac{3\totalsampD}{\confidence}}}} \leq \frac{\confidence}{3\totalsampD}
\end{equation}

So, if  $\distF^2(f_\surveyset, \optimalf) \leq \tol$  and $|\regnoisevariance - \hat{L}_S| \leq  0.1\tol$ then the probability that \SurVerif{} output \textbf{REJECT} in the \textbf{while} loop is at most $\delta/3$. Also, at the end of the \textbf{while} loop let $\estimatevarD_{\tau}$ be the value of $\estimatevarD$. The value of $\tau$ has been so chosen that 
 \begin{equation}\label{eq:final}
    \Pr \tbrac{\abs{\estimatevarD_{\tau} - \tau\E_{ \D}\tbrac{(f_\surveyset(\covariate) - \response)^2}} > 1.9\tau\tol} \leq 2\exp{\fbrac{-\frac{2\tau^2(1.9\tol)^2}{4\tau}}} = \frac{\confidence}{3}.
\end{equation}
Combining Equation~\eqref{eq:final} and \eqref{ineq:var} we see that if $ \distF^2(f_\surveyset, \optimalf) \leq \tol$ and $|\regnoisevariance - \hat{L}_S| \leq  0.1\tol$ then 
\begin{equation}\label{eq:final2}
\Pr \left[ \estimatevarD_{\tau} - \tau\hat{L}_S > 3\tau\tol \right] \leq \frac{\confidence}{3}.
\end{equation}

Finally, combining Equation~\eqref{eq:final2}, \eqref{eq:while} and \eqref{ineq: variance error} we have that if $\distF^2(f_\surveyset, \optimalf) \leq \tol$ then probability that \SurVerif{} outputs \textbf{REJECTS} is $\confidence$.

\ 



\textbf{Proof of \ref{subTheorem: far statement}: }
The proof of this part is simpler than the proof of \ref{subTheorem: close statement}. If $\distF^2(f_\surveyset, \optimalf) \geq 5\tol$ then we show that \SurVerif{} output \textbf{ACCEPT} in the final \textbf{if} statement is less than $\confidence$.  By Hoeffding's inequality we have the Equation~\eqref{eq:final}.
Combining Equation~\eqref{eq:final} with Lemma~\ref{lem:dist} and Equation~\eqref{ineq: variance error}  we see that if $\distF^2(f_\surveyset, \optimalf) \geq 5\tol$
then \begin{equation}\label{eq:final3}
\Pr[\estimatevarD_{\tau} - \tau\hat{L}_S \leq 3\tau\tol] \leq \frac{2\confidence}{3} < \confidence.
\end{equation}
This completes the proof.   
\end{proof}

\newpage
\section{Lower Bound for Model Reconstruction}

The task of checking if the regression coefficient for the data in $\surveyset$ is close to the regression coefficient for $\D$ can be checked directly by generating an estimate $\empcoeff$ of the optimal regression coefficient $\truecoeff$ corresponding to $\D$. However, the number of samples required for this approximate recovery problem grows with the dimension of the data. The following lemma, due to~\cite{DuchiWainwright/ArXiv/2013/ContinumFano} quantifies this dependence:

\begin{lemma}[~\cite{DuchiWainwright/ArXiv/2013/ContinumFano}]
\label{lemma: abc}
    For a regression model $\response = \inn{\truecoeff}{\covariate} +\regnoise$ with $\regnoise \sim \normal\fbrac{0,\regnoisevariance}$ for $\covariate,\truecoeff \in \R^\d$ with $\d \geq 2$, any algorithm that produces an estimate $\empcoeff$ of $\truecoeff$ using $\n$ samples must satisfy:
    \begin{align*}
        \norm{\truecoeff-\empcoeff}_2^2 \geq \frac{1}{32}\frac{d^2\regnoisevariance}{\norm{\covariatematrix}_{2,2}^2}
    \end{align*}
    In particular, If Assumption~\ref{Assumption: Bounded Variables} holds, we have:
    \begin{align*}
        \norm{\truecoeff-\empcoeff}_2^2 \geq \frac{d\regnoisevariance}{32\n}
    \end{align*}
\end{lemma}

We now provide the proof for Lemma~\ref{Lemma: Lower Bound for Reconstruction}, restated here.

\LBReconst*

\begin{proof}
Let $\errcoeff \coloneqq \truecoeff - \empcoeff$. Then for any $\covariate \in \R^d$, the difference between the true and estimated predictions is:
    \[
    \inn{\truecoeff}{\covariate}-\inn{\empcoeff}{\covariate} = \inn{\truecoeff-\empcoeff}{\covariate} = \inn{\errcoeff}{\covariate}
    \]
    Therefore, 
    
    \begin{align*}
        \E_{\covariate\sim\D}\tbrac{\inn{\errcoeff}{\covariate}^2} &=\E_{\covariate\sim\D}\tbrac{\errcoeff^T\covariate\covariate^T\errcoeff}\\
        &=\errcoeff^T\E_{\covariate\sim\D}\tbrac{\covariate\covariate^T}\errcoeff\\
        &\geq\mineigen\fbrac{\E_{\covariate\sim\D}\tbrac{\covariate\covariate^T}}\norm{\errcoeff}^2_2\\
        &\geq\mineigen\fbrac{\E_{\covariate\sim\D}\tbrac{\covariate\covariate^T}}\cdot\frac{d\regnoisevariance}{32\n} &&\text{From Lemma~\ref{lemma: abc}}\\
        &=\frac{\mineigen\fbrac{\mathrm{Cov}\fbrac{\covariate}} \d \regnoisesd^2}{32\n}
    \end{align*}
    
    Now, by setting the distance $\sqrt{\E_{\covariate\sim\D}\tbrac{\inn{\errcoeff}{\covariate}^2}}$ to be less than or equal to $\epsilon$, we get
    
    $$\n \geq \frac{\mineigen \d \regnoisesd^2}{32\epsilon^2}\,.$$
\end{proof}

Hence, we use the loss to identify the model distance between these two quantities. The loss of the two regressions follow two gaussians with different means and same variance. Here, we state a lower bound on the difference of means in this setup.






\newpage
\section{Experimental Results}

In this section, we detail the outcomes of our experiments described in Section~\ref{section:experiment}. In Table~\ref{tab:SynthRidge}~and~\ref{tab:SynthLasso}, we list the outcomes of \SurVerif{} on the synthetic dataset w.r.t. the $\ridgeset$ and $\lassoset$ model classes, respectively. In Table~\ref{tab:ACSRidge}~and~\ref{tab:ACSLasso}, we list the outcomes of \SurVerif{} on \acsincome{} dataset w.r.t. the $\ridgeset$ and $\lassoset$ model classes, respectively.  As stated in Section~\ref{section:experiment}, we have run 50 trials for all parameter choices, i.e. each row in the tables. The $\confidence$ is set to $0.01$ throughout. We also reproduce the figures here for ease of reading. 
The \red{red} and \blue{blue} lines represent the values of the \red{red} and \blue{blue} lines of their respective plots as defined in Section~\ref{section:experiment}.

\begin{figure*}[ht!]
\centering%
\begin{minipage}{\textwidth}
\centering
\includegraphics[scale=.5, width=.8\linewidth]{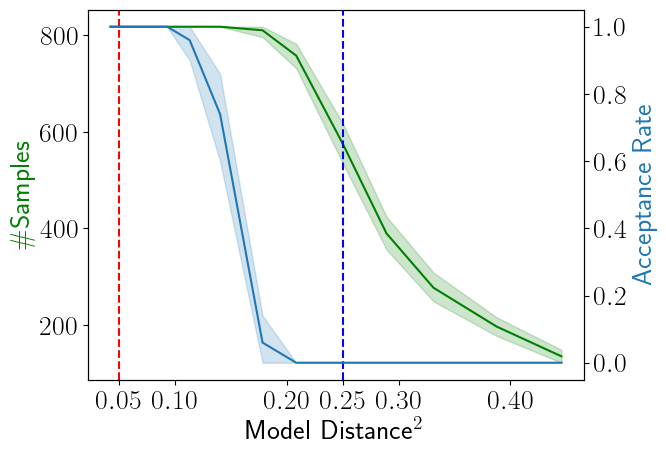}\vspace*{-.5em}
\caption{\small{Acceptance rate of \SurVerif{} w.r.t. model class $\ridgeset$ on Synthetic Data vs. change in $\mu$ (over 50 runs) for $\delta = 0.1$ and $\epsilon = 0.05$.}}\label{Fig: SynthRidge2}
\end{minipage}\hfill
\begin{minipage}{\textwidth}
\centering
\includegraphics[scale=.5, width=.8\linewidth]{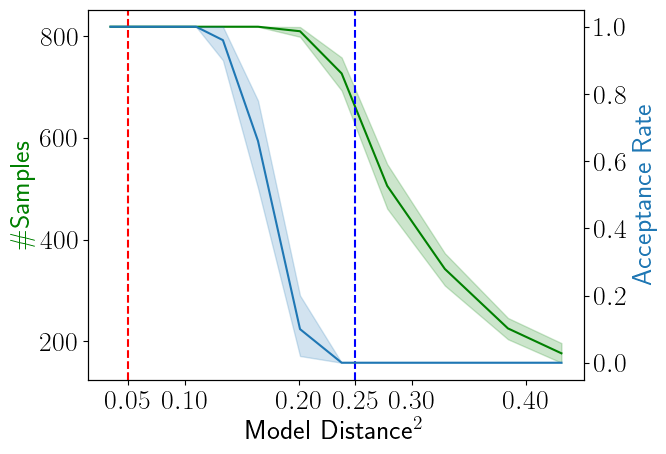}\vspace*{-.5em}
\caption{\small{Acceptance rate of \SurVerif{} w.r.t. model class $\lassoset$ on Synthetic Data vs. change in $\mu$ (over 50 runs) for $\delta = 0.1$ and $\epsilon = 0.05$.}}\label{Fig: SynthLASSO2}
\end{minipage}
\end{figure*}

\begin{table}[ht!]
    \centering
    \caption{Performance of \SurVerif{} on Synthetic Data w.r.t. $\ridgeset$ [Figure~\ref{Fig: SynthRidge2}]}
    \begin{tabular}{cccccc}
        \hline
        $\approxerror$ & \distname & Acceptance Rate & \#Avg. Samples Used & $\totalsampD$ & Early Rejection Ratio \\
        \hline
        0.05   & 0.04  & 1     & 818  & 818   & 1     \\
        0.05   & 0.04  & 1     & 818  & 818   & 1     \\
        0.05   & 0.05  & 1     & 818  & 818   & 1     \\
        \arrayrulecolor{red}\hline
        0.05   & 0.06  & 1     & 818  & 818   & 1     \\
        0.05   & 0.07  & 1     & 818  & 818   & 1     \\
        0.05   & 0.09  & 1     & 818  & 818   & 1     \\
        0.05   & 0.11  & 0.96  & 818  & 818   & 1     \\
        0.05   & 0.14  & 0.74  & 818  & 818   & 1     \\
        0.05   & 0.18  & 0.06  & 811  & 818   & 0.991 \\
        0.05   & 0.21  & 0     & 758  & 818   & 0.927 \\
        0.05   & 0.25  & 0     & 575  & 818   & 0.703 \\
        \arrayrulecolor{blue}\hline
        0.05   & 0.29  & 0     & 390  & 818   & 0.477 \\
        0.05   & 0.33  & 0     & 277  & 818   & 0.339 \\
        0.05   & 0.39  & 0     & 197  & 818   & 0.241 \\
        0.05   & 0.45  & 0     & 135  & 818   & 0.165 \\
        0.05   & 0.51  & 0     & 100  & 818   & 0.122 \\
        0.05   & 0.57  & 0     & 83   & 818   & 0.101 \\
        0.05   & 0.63  & 0     & 64   & 818   & 0.079 \\
        0.05   & 0.70  & 0     & 52   & 818   & 0.063 \\
        0.05   & 0.78  & 0     & 44   & 818   & 0.053 \\
        0.05   & 0.86  & 0     & 28   & 818   & 0.035 \\
        0.05   & 0.94  & 0     & 29   & 818   & 0.035 \\
        0.05   & 1.02  & 0     & 27   & 818   & 0.034 \\
        0.05   & 1.12  & 0     & 19   & 818   & 0.024 \\
        0.05   & 1.20  & 0     & 18   & 818   & 0.022 \\
        0.05   & 1.31  & 0     & 16   & 818   & 0.019 \\
        0.05   & 1.43  & 0     & 13   & 818   & 0.016 \\
        0.05   & 1.54  & 0     & 10   & 818   & 0.013 \\
        0.05   & 1.62  & 0     & 10   & 818   & 0.012 \\
        0.05   & 1.76  & 0     & 10   & 818   & 0.012 \\
        0.05   & 1.88  & 0     & 9    & 818   & 0.011 \\
        \arrayrulecolor{black}\hline
    \end{tabular}
    \label{tab:SynthRidge}
\end{table}

\begin{table}[ht!]
    \centering
    \caption{Performance of \SurVerif{} on Synthetic Data w.r.t. $\lassoset$ [Figure~\ref{Fig: SynthLASSO2}]}
    \begin{tabular}{cccccc}
        \hline
        $\approxerror$ & \distname & Acceptance Rate & \#Avg. Samples Used & $\totalsampD$ & Early Rejection Ratio \\
        \hline
        0.05  & 0.03  & 1     & 818  & 818   & 1     \\
        0.05  & 0.04  & 1     & 818  & 818   & 1     \\
        0.05  & 0.04  & 1     & 818  & 818   & 1     \\
        0.05  & 0.05  & 1     & 818  & 818   & 1     \\
        \arrayrulecolor{red}\hline
        0.05  & 0.07  & 1     & 818  & 818   & 1     \\
        0.05  & 0.09  & 1     & 818  & 818   & 1     \\
        0.05  & 0.11  & 1     & 818  & 818   & 1     \\
        0.05  & 0.13  & 0.96  & 818  & 818   & 1     \\
        0.05  & 0.16  & 0.66  & 818  & 818   & 1     \\
        0.05  & 0.20  & 0.1   & 809  & 818   & 0.989 \\
        0.05  & 0.24  & 0     & 726  & 818   & 0.888 \\
                \arrayrulecolor{blue}\hline
        0.05  & 0.28  & 0     & 506  & 818   & 0.618 \\
        0.05  & 0.33  & 0     & 342  & 818   & 0.418 \\
        0.05  & 0.38  & 0     & 225  & 818   & 0.276 \\
        0.05  & 0.43  & 0     & 177  & 818   & 0.216 \\
        0.05  & 0.50  & 0     & 122  & 818   & 0.150 \\
        0.05  & 0.56  & 0     & 89   & 818   & 0.109 \\
        0.05  & 0.62  & 0     & 68   & 818   & 0.083 \\
        0.05  & 0.70  & 0     & 54   & 818   & 0.066 \\
        0.05  & 0.78  & 0     & 51   & 818   & 0.063 \\
        0.05  & 0.85  & 0     & 36   & 818   & 0.044 \\
        0.05  & 0.93  & 0     & 27   & 818   & 0.032 \\
        0.05  & 1.02  & 0     & 21   & 818   & 0.026 \\
        0.05  & 1.10  & 0     & 21   & 818   & 0.025 \\
        0.05  & 1.21  & 0     & 16   & 818   & 0.020 \\
        0.05  & 1.31  & 0     & 14   & 818   & 0.017 \\
        0.05  & 1.43  & 0     & 14   & 818   & 0.017 \\
        0.05  & 1.51  & 0     & 13   & 818   & 0.016 \\
        0.05  & 1.63  & 0     & 11   & 818   & 0.013 \\
        0.05  & 1.75  & 0     & 9    & 818   & 0.011 \\
        0.05  & 1.87  & 0     & 9    & 818   & 0.011 \\
                \arrayrulecolor{black}\hline
    \end{tabular}
    \label{tab:SynthLasso}
\end{table}

\clearpage

\begin{figure*}[ht!]
\begin{minipage}{\textwidth}
\centering
\includegraphics[scale=.5, width=.8\linewidth]
{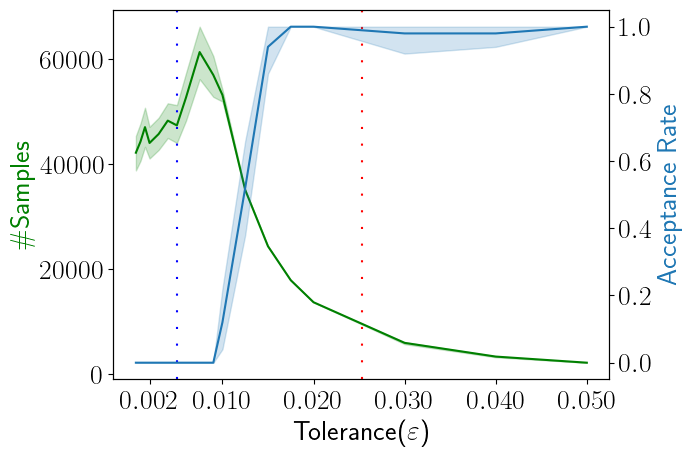}\vspace*{-.5em}
\caption{\small{Acceptance rate of \SurVerif{} w.r.t. model class $\ridgeset$ on \acsincome~ (over 50 runs) for $\delta = 0.1$ and varying range of $\epsilon$.}}\label{Fig: ACSRidge2}
\end{minipage}
\end{figure*}

\begin{figure*}[ht!]
\begin{minipage}{\textwidth}
\centering
\includegraphics[scale=.5, width=.8\linewidth]{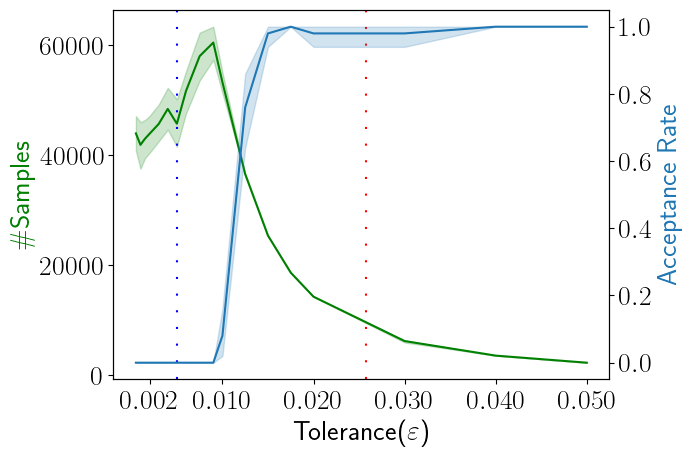}\vspace*{-.5em}
\caption{\small{Acceptance rate of \SurVerif{} w.r.t. model class $\lassoset$ on \acsincome~ (over 50 runs) for $\delta = 0.1$ and varying range of $\epsilon$.}}\label{Fig: ACSLasso2}
\end{minipage}
\end{figure*}

\begin{table}[ht!]
    \centering
    \caption{Performance of \SurVerif{} on \acsincome{} w.r.t. $\ridgeset$ [Figure~\ref{Fig: ACSRidge2}]}
    \begin{tabular}{cccccc}
        \hline
        $\approxerror$ & \distname & Acceptance Rate & \#Avg. Samples Used & $\totalsampD$ & Early Rejection Ratio\\
        \hline
        0.05    & 0.0258 & 1    & 2195   & 2195     & 1.000  \\
        0.04    & 0.0258 & 0.98 & 3361   & 3430     & 0.980  \\
        0.03    & 0.0258 & 0.98 & 5975   & 6097     & 0.980  \\
        \arrayrulecolor{red}\hline
        0.02    & 0.0258 & 1    & 13717  & 13717    & 1.000  \\
        0.0175  & 0.0258 & 1    & 17916  & 17916    & 1.000  \\
        0.015   & 0.0258 & 0.94 & 24386  & 24386    & 1.000  \\
        0.0125  & 0.0258 & 0.52 & 35115  & 35115    & 0.933  \\
        0.01    & 0.0258 & 0.12 & 53222  & 54867    & 0.858  \\
        0.009   & 0.0258 & 0    & 57022  & 67737    & 0.571  \\
        0.0075  & 0.0258 & 0    & 61388  & 97542    & 0.326  \\
        0.006   & 0.0258 & 0    & 52710  & 152409   & 0.200  \\
        \arrayrulecolor{blue}\hline
        0.005   & 0.0258 & 0    & 47414  & 219468   & 0.136  \\
        0.004   & 0.0258 & 0    & 48322  & 342919   & 0.072  \\
        0.003   & 0.0258 & 0    & 45773  & 609633   & 0.031  \\
        0.002   & 0.0258 & 0    & 44064  & 1371675  & 0.017  \\
        0.0015  & 0.0258 & 0    & 47085  & 2438532  & 0.008  \\
        0.001   & 0.0258 & 0    & 44301  & 5486697  & 0.007  \\
        0.0005  & 0.0258 & 0    & 42207  & 21946787 & 0.002  \\
       \arrayrulecolor{black}\hline
    \end{tabular}
    \label{tab:ACSRidge}
\end{table}

\begin{table}[ht!]
    \centering
    \caption{Performance of \SurVerif{} on \acsincome{} w.r.t. $\lassoset$ [Figure~\ref{Fig: ACSLasso2}]}
    \begin{tabular}{cccccc}
        \hline
        $\approxerror$ & \distname & Acceptance Rate & \#Avg. Samples Used & $\totalsampD$  & Early Rejection Ratio \\
        \hline
        0.05    & 0.0259 & 1    & 2285   & 2285     & 1.000 \\
        0.04    & 0.0259 & 1    & 3570   & 3570     & 1.000 \\
        0.03    & 0.0259 & 0.98 & 6219   & 6346     & 0.980 \\
         \arrayrulecolor{red}\hline
        0.02    & 0.0259 & 0.98 & 14279  & 14279    & 1.000 \\
        0.0175  & 0.0259 & 1    & 18650  & 18650    & 1.000 \\
        0.015   & 0.0259 & 0.98 & 25384  & 25384    & 1.000 \\
        0.0125  & 0.0259 & 0.76 & 36551  & 36553    & 1.000 \\
        0.01    & 0.0259 & 0.08 & 53262  & 57114    & 0.933 \\
        0.009   & 0.0259 & 0    & 60467  & 70511    & 0.858 \\
        0.0075  & 0.0259 & 0    & 57999  & 101535   & 0.571 \\
        0.006   & 0.0259 & 0    & 51660  & 158649   & 0.326 \\
        \arrayrulecolor{blue}\hline
        0.005   & 0.0259 & 0    & 45721  & 228454   & 0.200 \\
        0.004   & 0.0259 & 0    & 48408  & 356959   & 0.136 \\
        0.003   & 0.0259 & 0    & 45663  & 634593   & 0.072 \\
        0.002   & 0.0259 & 0    & 43893  & 1427833  & 0.031 \\
        0.0015  & 0.0259 & 0    & 42993  & 2538370  & 0.017 \\
        0.001   & 0.0259 & 0    & 41880  & 5711332  & 0.007 \\
        0.0005  & 0.0259 & 0    & 43961  & 22845325 & 0.002 \\
        \arrayrulecolor{black}\hline
    \end{tabular}
    \label{tab:ACSLasso}
\end{table}



\end{document}